\documentclass[11pt]{article}
\usepackage{graphicx}

\usepackage{indentfirst,csquotes}

\usepackage{amsmath,amsfonts,bm}

\def\eqref#1{equation~\ref{#1}}

\def\1{\bm{1}}

\DeclareMathAlphabet{\mathsfit}{\encodingdefault}{\sfdefault}{m}{sl}
\SetMathAlphabet{\mathsfit}{bold}{\encodingdefault}{\sfdefault}{bx}{n}

\newcommand{\E}{\mathbb{E}}

\newcommand{\KL}{D_{\mathrm{KL}}}
\newcommand{\Var}{\mathrm{Var}}

\usepackage[letterpaper,margin=1in,headheight=15pt]{geometry}

\usepackage{fancyhdr}
\usepackage[bottom]{footmisc}

\usepackage{url}            
\usepackage{booktabs}       
\usepackage{amsfonts, amsmath}       
\usepackage{nicefrac}       
\usepackage{microtype}      
\usepackage{xcolor}         

\usepackage{nicefrac}

\usepackage{amsmath,amssymb,amsthm,mathtools,mathrsfs}
\usepackage{algorithm,algpseudocode}
\usepackage{booktabs,tabularx,threeparttable}
\usepackage{hyperref}       
\hypersetup{
	colorlinks   = true, 
	urlcolor     = blue, 
	linkcolor    = blue, 
	citecolor   = blue 
}
\usepackage{multirow}
\usepackage{csquotes}
\usepackage{enumitem}
\setlist{nosep,leftmargin=*}
\newtheorem{theorem}{Theorem}[section]
\newtheorem{lemma}[theorem]{Lemma}
\newtheorem{corollary}[theorem]{Corollary}

\newtheorem{definition}[theorem]{Definition}

\theoremstyle{remark}

\usepackage{chngcntr}
\usepackage{apptools}

\newtheorem*{lemma*}{Lemma}

\newtheorem*{Proposition*}{Proposition}
\usepackage{graphicx}
\usepackage{graphics}
\usepackage{parskip}
\usepackage{amsmath}

\usepackage{subfig}

\usepackage{dsfont}
\usepackage{natbib}
\usepackage[leftcaption]{sidecap}

\usepackage{booktabs,tabularx,threeparttable}

\newcommand{\Astar}{\mathcal A^\star}
\newcommand{\F}{\mathcal F}
\newcommand{\Prob}{\mathbb P}
\newcommand{\ind}{\mathbf 1}
\newcommand{\clip}{\operatorname{clip}}
\newcommand{\kl}{\operatorname{kl}}
\newcommand{\Alt}{\operatorname{Alt}}

\usepackage{wrapfig}
\usepackage[normalem]{ulem}
\usepackage{soul}
\usepackage[normalem]{ulem}               
\newcommand\redout{\bgroup\markoverwith
{\textcolor{red}{\rule[0.5ex]{2pt}{0.8pt}}}\ULon}

\title{Balanced Sampling in Stochastic Multi-objective Bandits}

\author{%
  Mengfan Xu\thanks{Generative AI tools, particularly GPT-6, assisted in the development
of this manuscript. Details of their use are disclosed in Section \ref{final}.
}\\
  Mechanical and Industrial Engineering \\
  University of Massachusetts Amherst \\
  \texttt{mengfanxu@umass.edu} \\
}

\date{October 26, 2026}

\begin{document}
\maketitle

\begin{abstract}
\noindent In stochastic multi-objective bandits, Pareto optimal arms may represent
different trade-offs among objectives. Balanced sampling across them is
important in practical settings such as online advertising. Yet, there seems to be 
little consensus on how to measure fairness, to our best knowledge. Achieving fairness is also tough as the set of Pareto optimal arms is unknown, and early sampling decisions
may create imbalances that are difficult to erase afterwards. We measure unfairness by the KL divergence from the uniform
distribution over the true Pareto set to the empirical measure induced
by the sampling trajectory. Nicely, this measure admits an exact
decomposition that separates the cost of sampling suboptimal arms from
the imbalance among Pareto optimal arms. With a recursive form for unfairness, we develop
\emph{adaptive top-two sampling}, which combines sampling counts
and empirical evidence of Pareto membership to define a gain
for each arm and guide sampling without committing to an
estimated Pareto set. When the Pareto set is stable under small perturbations, we establish
an $O(\log T/T)$ and $O(T^{-1})$ bound on expected unfairness and variance, respectively. 
Information-theoretic lower bounds show that the $\log T/T$ rate
can be unavoidable for uniformly fair policies on stable instances,
matching our upper bound. More surprisingly, stability is indeed crucial, as the expected unfairness can be $\omega(\log T/T)$ without it. Finally, we create 
a variant that encourages exploration of undersampled arms to reduce
the risk of persistent underallocation. At sufficiently large horizons,
it preserves the rate on the expected unfairness while admitting a sharper bound of $O(\log^2 T/T^2)$ on variance.
Numerical studies support these findings and suggest
practical insights.\end{abstract}

\section{Introduction}
\label{sec:introduction}

Multi-armed bandits have supported sequential decision making widely, where a decision maker repeatedly selects an option (arm) and receives information about its performance in the form of a reward \citep{auer2002finite,auer2002nonstochastic,bubeck2012regret}. Classical bandits typically consider a scalar reward and measure performance by regret, which captures the cumulative cost of selecting suboptimal arms. However, many decisions involve multiple objectives beyond a single scalar reward. Multi-objective bandits (MO-MAB) have emerged to address such problems \citep{drugan2013designing,busa2017multi,mehrotra2020bandit,lu2019multi}. Here, each arm exhibits a vector of rewards, and Pareto dominance provides a new notion of optimality. An arm is Pareto optimal if no other arm performs at least as well in every objective and strictly better in at least one. 

Importantly, multiple arms may be Pareto optimal, and the set of such arms is called the Pareto set by convention.  In fact, different Pareto optimal arms may
represent different trade-offs among objectives. For example,
in online recommendation, different choices may favor user
engagement or supplier diversity, while in network routing,
different routes may offer trade-offs between latency and
reliability \citep{mehrotra2020bandit,lu2019multi}.
Repeatedly selecting only a subset of Pareto optimal arms
may favor particular trade-offs, thereby incurring opportunity costs for others. A balanced selection across Pareto optimal arms helps to balance these different forms of trade-offs in real-world scenarios, and thus demands our attention. In fact, in classical bandits, fairness  \citep{wang2021fairness, singh2018fairness, simchi2023multi} has been well motivated and understood. Yet we can imagine that these insights may not translate into MO-MAB given Pareto dominance.

A substantial line of research on MO-MAB has studied how to learn and select Pareto optimal arms. Amongst, two most common goals are 1) to minimize Pareto regret, which measures the cumulative cost of selecting Pareto suboptimal arms \citep{drugan2013designing,turgay2018contextual,lu2019multi,guan2026stochastic}, and 2) to identify the Pareto set through careful sampling with as small sample complexity as possible  \citep{auer2016pareto,ararat2023vector,garivier2024sequential,kone2024bandit}. Yet, \emph{neither  guarantees a balanced (or fair) allocation among these Pareto optimal arms.} To see this, consider two arms with mean reward vectors $(1,0)$ and $(0,1)$. Both arms are Pareto optimal, and selecting either incurs zero Pareto regret. In fact, a learner may select the first arm almost exclusively and still achieve zero regret, although the second arm receives almost no observations. Pareto set identification, on the other hand, concerns which arms belong to the set and stops once we figure out both arms are Pareto optimal. By doing so, we emphasize their total sampling frequencies that \emph{sum} over arms and time, but not their \emph{relative} sampling frequencies. The latter indeed matters a lot when we care about the sampling trajectory \emph{across} arms and time steps. 

In parallel, fairness in MO-MAB has received relatively less attention, to the extent of our knowledge.
Existing work has approached it, but there has been a lack of consensus on which definition we
should adopt in general. The first debate concerns fairness
across objectives \citep{ann2026blessings} or among arms
\citep{yahyaa2014annealing,mehrotra2020bandit}.
The latter is more appropriate when we want to balance
the trade-offs among objectives entangled at the arm level
via future sampling, 
whereas the former is more about the disentangled trade-offs
among objectives when looking backward. Herein, we focus on the latter,
in line with classical bandits and our goal in an online setting.
Within this layer, multiple measures still exist, mostly based on
information entropy. On the positive side, this means that
researchers have already considered this problem; on the flip side,
it may imply that we have not yet converged on a coherent definition,
as we can still find cases where the existing measures are less
appropriate. For example, \citet{yahyaa2014annealing} use Shannon entropy \citep{sethna2006statistical} 
to measure unfairness, which increases as the allocation among the Pareto set becomes more uniform. Yet, uniform allocation can actually be
the fairest. A consequence of the lack of consensus is that there have been almost
no theoretical analyses of fairness in MO-MAB, to the best of our
knowledge.

Operationally, achieving balanced allocation is challenging because the Pareto set is unknown, while fairness requires controlling the sampling process over time. If the Pareto set were known, we could simply favor its least-sampled arms. In practice, however, we must infer Pareto membership from noisy observations while deciding how to allocate samples. A natural approach is to estimate the Pareto set and balance sampling within it. Yet, an incorrect estimate may have substantial consequences. A Pareto optimal arm mistakenly classified as dominated may receive few further observations, delaying the correction of its estimate and causing persistent underallocation. Such selections would increase the cumulative cost as well. Even if the Pareto set is eventually identified correctly, the learning process may have already consumed a substantial portion of the horizon. We may not have enough time to compensate for the underallocation of newly discovered Pareto optimal arms. This difficulty is particularly pronounced near a boundary of Pareto dominance, where small estimation errors can flip Pareto membership entirely. Thus, learning Pareto membership and balancing allocation necessitate a joint treatment, as each selection affects both. We hope to find out:
\begin{center}
\emph{Can we find a satisfying measure of fairness and achieve
fair allocation over an unknown Pareto set?}
\end{center}

\paragraph{Main contributions.} We answer this question firmly. We first measure terminal unfairness by the KL divergence
$\mathcal F_T=D_{\mathrm{KL}}(\pi^\star\|\widehat\pi_T)$,
where $\pi^\star$ is the uniform distribution over the true Pareto set
and $\widehat\pi_T$ is the empirical allocation. Intuitively, the closer the empirical allocation is to a uniform distribution across Pareto-optimal arms, the greater the fairness.
Interestingly, this measure admits an exact decomposition:
\begin{equation}
\label{eq:intro-decomp}
\mathcal F_T
=
-\log(1-\nicefrac{B_T}{T})
+
D_{\mathrm{KL}}(U_m\|q_T),
\end{equation}
where $B_T$ is the number of dominated-arm pulls, $m$ is the Pareto
set size, and $q_T$ is the empirical allocation conditional on
Pareto-optimal pulls. The two terms capture the cost of sampling
dominated arms and the imbalance among Pareto-optimal arms,
respectively. This provides a natural connection between fairness
and Pareto regret while revealing the additional requirement of
balanced allocation. After careful comparison with existing measures,
we find that this measure is coherent and appropriate while preserving
the information-theoretic structure.

The one-step change in unfairness shows that the benefit of sampling
an arm depends on both its sampling count and Pareto membership.
To this end, we develop \emph{adaptive top-two sampling}, which
defines a step-wise fairness gain based on the sampling count and
an exponential weight reflecting empirical evidence of Pareto membership.
It selects between the two arms with the largest gains,
promoting exploration without committing to an estimated Pareto set. When the Pareto set is stable under small changes in mean rewards, we show that unfairness has mean $O_\mu(\log T/T)$ and
variance $O_\mu(T^{-1})$. Moreover, as a by-product, the expected Pareto regret grows
at most logarithmically. With information-theoretic lower bounds, we show that
the $\log T/T$ rate can be unavoidable for uniformly fair policies on instances where the Pareto set is stable.
More surprisingly, at a boundary where the Pareto set is unstable, we
prove a cost of $\omega(\log T/T)$ for any uniformly fair policy, which also justifies the assumption of stability. These results speak to the power of Pareto geometry.

Finally, a variance of order $O_\mu(T^{-1})$ could require
$\Omega(T)$ independent runs to estimate expected fairness
with sufficient precision. To improve reliability across
trajectories, we introduce a method
that encourages more exploration of undersampled arms and thus reduces mistakes about Pareto membership that can lead to catastrophic unfairness. When the horizon is sufficiently large, it
preserves the $O_\mu(\log T/T)$ rate for the expected unfairness while
admitting a sharper upper bound of
$O_\mu(\log^2 T/T^2)$ on variance. At smaller horizons, however, it may increase pulls of Pareto suboptimal arms, suggesting a subtle trade-off whose consequences
depend on the stability of the Pareto set.

\section{Problem Formulation}\label{sec:model}
We consider a stochastic multi-objective bandit problem, which itself has been by large studied in recent years, but from a new perspective of balanced allocation across good arms that first demands a new framework for measuring "being balanced." This section is thus devoted to such conceptualization. Let us consider $K\ge2$ arms and $d\ge1$ reward coordinates for each arm $i \in [K]$. Each arm is associated with an independent reward tape $X_i = (X_i^{(1)}, \ldots, X_i^{(d)})$ which is independent across arms and time steps and identically distributed across time; however, different coordinates of a reward vector of the same arm may be dependent. The reward vector of arm $i$ has a unknown mean vector $\mu_i=(\mu_i^{(1)},\ldots,\mu_i^{(d)})\in\mathcal M$, where $\mathcal M=\prod_{r=1}^d[a_r,b_r]$ is a Cartesian product of the coordinate ranges. Without loss of generality, we assume $a_r,b_r$ are known; typically $a_r = 0, b_r=1$ in the literature. Every reward coordinate $X_i^{(r)}$ follows a $\sigma^2$-sub-Gaussian distribution with known variance proxy $\sigma^2$, formally 
\begin{equation}
\E\exp\bigl(\lambda(X_i^{(r)}-\mu_i^{(r)})\bigr)\le \exp(\sigma^2\lambda^2/2),\qquad \lambda\in\mathbb R.
\end{equation}   Let us denote the time horizon of the bandit problem by $T$. At time $t \in [T]$, a player, equipped with a (possibly random) policy, chooses arm $A_t$ based on what it has observed thus far without any other information beyond the history. It then receives a reward vector $X_{A_t, t}$ which is one random realization of $X_{A_t}$ of the selected arm $A_t$ at time $t$.   

The goal of the player is to select the "good" arms in a "balanced" way, and let us formalize these terminologies one by one. We measure "good" through the celebrated Pareto dominance \citep{drugan2013designing, mehrotra2020bandit}. An arm $j$ \emph{dominates} arm $i$ when $\mu_j^{(r)}\ge\mu_i^{(r)}$ for every $r$, with at least one strict inequality. Let $\Astar=\Astar(\mu)$ be the nondominated arm set, and thus represent the set of "good" arms, which is oftentimes called Pareto optimal set; let us also denote its cardinality by $m=|\Astar|$.

A learning policy $\pi=(\pi_t)_{t\ge1}$ specifies
a probability distribution $\pi_t(\cdot\mid H_{t-1})$
on the measurable space of arms $([K],2^{[K]})$
at each time $t$, conditional on the observed history
$H_{t-1}=(A_s,X_{A_s,s})_{s=1}^{t-1}$.
Jointly, the policy and reward 
induce a probability measure $\mathbb P_{X}^\pi$
over the sequence of actions and rewards.

Let $N_i(t)=\sum_{s=1}^t\ind\{A_s=i\}$
denote the total number of pulls of arm $i$
up to and including time $t$. The empirical
frequency of arm $i$ is then
$\widehat\pi_{t,i}=N_i(t)/t$.
From a probabilistic perspective, these frequencies
define an empirical probability measure
$\widehat\pi_t$ on $([K],2^{[K]})$, with
$\widehat\pi_t(\{i\})=\widehat\pi_{t,i}$.
This is a random probability measure constructed
from the action sequence $(A_1,\ldots,A_t)$,
whose random variables are generally dependent
and non-identically distributed, with randomness
governed by $\mathbb P_{X}^\pi$.

Subsequently, we are clear that the target of the player is the ideal measure defined through $\pi_i^\star=m^{-1}\ind\{i\in\Astar\}$, which assigns equal probabilities to all Pareto optimal arms and 0 probabilities to suboptimal arms, though of course $\Astar$ is unknown to us as the mean vectors are unknown. This language in probability motivates and also prepares us to measure \emph{terminal unfairness} by the KL divergence between the ideal measure and the empirical measure 
\begin{equation}\label{eq:fairness}
\F_T=\KL(\pi^\star\|\widehat\pi_T)
=\frac1m\sum_{i\in\Astar}\log\frac{T}{mN_i(T)}.
\end{equation}
Naturally, we can also consider the unfairness of the expected empirical measure  where $\E$ is over the randomness of the allocation policy itself, \[\bar{\F}_T = \KL(\pi^\star\|\E\widehat\pi_T) = \frac1m\sum_{i\in\Astar}\log\frac{T}{m\E[N_i(T)]}.\] By the Jensen's inequality, we have that 
\(
E[\mathcal{F}_T] \geq \bar{\mathcal{F}}_T. 
\)

We would like to compare our measure with existing measures of fairness across Pareto optimal arms. These differ from measures of fairness across objectives, so for the sake of space and relevance, we omit that discussion here and refer readers to \citep{busa2017multi, mehrotra2020bandit, ann2026blessings} for a detailed introduction. It is surprising that our fairness also presents as an upper bound for the objective fairness in \citep{ann2026blessings}, and thus provides a natural bridge between objective fairness and fairness over Pareto optimal arms herein.  Note that an existing measure of unfairness measures the KL divergence between the optimal sampling policy and the actual sampling policy \citep{yahyaa2014annealing}. However, it does not restrict arms to be Pareto optimal and is ill-defined if we consider only Pareto optimal arms. The KL measure coincides with our measure when the target distribution is uniform over all Pareto optimal arms and all arms are Pareto optimal. Another existing measure is based on Shannon entropy \citep{yahyaa2014annealing}, namely, the entropy of the actual sampling frequencies restricted to Pareto optimal arms, which is essentially the empirical measure of the conditional distribution of the sampling policy. Higher entropy represents greater unfairness under their definition. This is fundamentally flawed, however, because it is well known that the uniform distribution has maximum entropy among all discrete distributions and, based on their definition, represents the most unfair allocation. At the same time, the uniform distribution is the fairest allocation when there is no preference over the priority of the Pareto optimal arms. This contradiction suggests that it may be less appropriate to use this measure. Instead, we directly compare the empirical measure with the fairest allocation, that is, the uniform distribution, and a smaller distance indicates a fairer allocation. In the meantime, our measure still admits an entropy formulation through the relationship between KL divergence and Shannon entropy, which we believe is satisfying.

More broadly, while we use the KL divergence distance, many other distance measures from probability theory can be adopted. Examples include the Kolmogorov, Wasserstein, and total variation distances. Our result provides a conservative guarantee, as the Kolmogorov distance is bounded by total variation, and the Wasserstein distance is bounded by a constant multiple of total variation on a finite space with a bounded ground metric \citep{ross2011fundamentals}. Moreover, total variation is bounded by the square root of half the KL divergence through Pinsker's inequality. Thus, by providing an upper bound on the KL divergence that defines balancedness, we also obtain upper bounds on all these distance measures in some sense. 

Beyond its role as a proper measure of balancedness, surprisingly, it can also connect well to existing metrics of efficiency of the selected arms in multi-objective bandits. Taking a closer look at this metric, we find that it admits a useful exact decomposition. To see this, let $B_T=\sum_{i\notin\Astar}N_i(T)$ denote the total number of pulls of dominated arms, which we refer to the less efficient pulls, and let $H_T=T-B_T$ be the total number of pulls of Pareto optimal arms. For each $i\in\Astar$, let us define $q_{T,i}=N_i(T)/H_T$ as its sampling proportion conditional on pulling a Pareto optimal arm. Finally, let $U_m$ denote the uniform distribution on $m$ points.

Throughout, we assume that $N_i(T) \geq 2$ for all arm $i \in [K]$ and $T \geq 2K$. This is concerned about the initialization, where the policy pulls each arm at least twice; we note that it is quite common to pull arms at least once, and this twice condition, though requiring a bit more, is still mild. 
\begin{lemma}\label{lem:decomp}
For every trajectory with positive counts on $\Astar$, which automatically holds by the condition of $N_i(T) \geq 2$ for every $i$, we have 
\begin{align}\label{eq:decomp}
\F_T=-\log(1-B_T/T)+\KL(U_m\|q_T), \KL(U_m\|q_T)\le \frac1{H_T}\sum_{i\in\Astar}\frac{(N_i(T)-H_T/m)^2}{N_i(T)}
\end{align}
In particular, $\F_T\ge B_T/T$ and $0\le\F_T\le\log(T/m)$. 
\end{lemma}
\begin{proof}
Let us first substitute $N_i(T)/T=(H_T/T)q_{T,i}$ into \eqref{eq:fairness}. By the relationship  $\KL(P\|Q)\le\chi^2(P\|Q)$, with $P=U_m$ and $Q=q_T$, we obtain that 
$\chi^2(U_m\|q_T)=\sum_i(1/m-q_{T,i})^2/q_{T,i}=H_T^{-1}\sum_i (N_i(T)-H_T/m)^2/N_i(T)$. Lastly, the bound $\F_T\le\log(T/m)$ follows from $N_i(T)\ge2$. It is worth noting that both terms in \eqref{eq:decomp} are nonnegative, and $-\log(1-z)\ge z$ which implies that $\F_T\ge B_T/T$. 
\end{proof}

The first term is closely related to the concept of Pareto regret \citep{drugan2013designing, xu2023pareto}, which penalizes pulls of dominated arms through a weighted sum of $N_i(T)$, where the weight is determined by arm $i$'s distance to the Pareto front. In contrast, this new metric penalizes $B_T$ through a logarithmic transformation. But interestingly, $-\log(1-B_T/T)\approx B_T/T$ when $B_T/T$ is small, so it resembles a normalized version of Pareto regret when all weights are equal. This decomposition is very useful in that it shows that this target not only requires the allocation to be as close to uniform across Pareto optimal arms as possible, but also necessitates allocating as much as possible to the Pareto optimal set of arms. Hence, optimizing this target gives the best of both worlds, in terms of both efficiency and balance, in a rough sense; later, we provide precise evidence for this.

\section{Methodology}

Toward this goal, we hope to design a new learning method, which is the focus of this section. Our starting point is the observation that the one-step change in unfairness satisfies
\[
\mathcal{F}_t-\mathcal{F}_{t-1}
=
\log\frac{t}{t-1}
-
\frac{\mathbb{I}\{A_t\in\mathcal{A}^{\star}\}}{m}
\log\left(
1+\frac{1}{N_{A_t}(t-1)}
\right).
\]
If the Pareto set were known, minimizing the one-step fairness loss would amount to maximizing
\[
\frac{\mathbb I\{i\in\mathcal A^\star\}}{m}
\log\left(1+\frac1{N_i(t)}\right).
\]
The arm selection then becomes quite straightforward. Interestingly, this greedy selection is closely related to the
Frank--Wolfe method \citep{jaggi2013revisiting, lacoste2015global} for constrained optimization. To see this,
we are essentially minimizing the KL divergence
$\mathcal F(\pi)=D_{\mathrm{KL}}(\pi^\star\|\pi)$
over the probability simplex on $([K],2^{[K]})$ where $\pi$ has to be an empirical measure. At each step, Frank--Wolfe
selects a vertex of the simplex that minimizes the first-order
approximation of the objective. Concretely, a vertex
$e_i$ corresponds to an arm $i$, and the partial derivative is exactly 
\[
\frac{\partial\mathcal F(\pi)}{\partial\pi_i}
=
-\frac{\mathbb I\{i\in\mathcal A^\star\}}{m\pi_i}.
\]
Thus, it amounts to
selecting the Pareto optimal arm with the smallest $\widehat\pi_{t,i}=N_i(t)/t$.
This coincides with our rule by noting that 
$\frac{\mathbb I\{i\in\mathcal A^\star\}}{m}
\log\left(1+\frac1{N_i(t)}\right) = \frac{1}{m}\log(1+1/N_i(t))$ for such arms.
Moreover, pulling arm $i$ updates the empirical allocation as
\[
\widehat\pi_{t+1}
=
\widehat\pi_t+\frac{1}{t+1}(e_i-\widehat\pi_t),
\]
which is precisely a Frank--Wolfe update with step size
$1/(t+1)$. It is inspiring to see this connection although we use the exact change whereas Frank--Wolfe uses 
the first-order approximation.

Unfortunately, Pareto membership is unknown, which is why this problem is hard and why we believe this work is worth developing. A direct idea is to estimate the Pareto front $\mathcal A^{\star}$, and replace the quantity $\mathcal A^{\star}$ in the above loss with this estimation. However, as we shall foresee, this estimation can be notoriously hard, and a wrong estimation of the Pareto membership of just one arm could flip the binary variable $I\{i\in\mathcal A^\star\}$ entirely and thus directly propagates to the loss estimation and further the decision making. Each observation affects not only the empirical allocation but also the learner's knowledge of the target distribution, resembling a connection to the classical knowledge-gradient principle and the value of information. To this end, we hope to construct a soft target
that incorporates both the empirical dominance margin and
the statistical evidence for Pareto membership. The soft target replaces this \emph{discrete weight} $I\{i\in\mathcal A^\star\}$ by a \emph{continuous score} using  dominance and statistical evidence of arms accumulated thus far, and we use the resulting fairness gains to guide arm sampling. 

More specifically, at time $t$ and for arm $i$, let $\widehat\mu_{t,i}$ be the empirical mean, and $\widetilde\mu_{t,i}=\Pi_{\mathcal M}(\widehat\mu_{t,i})$ its coordinate-wise projection onto $\mathcal M$.  We first measure how strongly arm $i$ appears to be dominated
through the empirical dominance margin
\begin{equation}\label{eq:Delta}
\widehat\Delta_{t,i}=
[\max_{j\ne i}\min_{r\in[d]}
(\widetilde\mu_{t,j}^{(r)}-\widetilde\mu_{t,i}^{(r)})]_+,
\quad
R_{t,i}=\nicefrac{\sqrt{N_i(t)}}{\sigma}\widehat\Delta_{t,i}.
\end{equation}
A positive $\widehat\Delta_{t,i}$ indicates that another arm
has a larger empirical mean in every objective and thus arm $i$ is empirically dominated; we would prefer arms with small $\widehat\Delta_{t,i}$. The quantity
$R_{t,i}$ further accounts for the number of observations
collected from arm $i$ in an effort to balance exploration (through $N_i(t)$) and exploitation (through $\widehat\Delta_{t,i}$). Importantly, we also need to take into account the sensitivity of Pareto membership to estimation error as these quantities are all empirical, and thus we thus adopt a soft-max style criteria as follows by incorporating a temperature. Specifically, for any
candidate instance $x\in\mathcal M^K$, let
\[
\Alt_i(x)=
\left\{\nu\in\mathcal M^K:
\ind\{i\in\Astar(\nu)\}
\ne\ind\{i\in\Astar(x)\}
\right\}
\]
denote the set of alternative instances under which arm $i$
has different Pareto membership. We then define
\begin{equation}\label{eq:Z}
Z_{t,i}=
\inf_{\nu\in\Alt_i(\widetilde\mu_t)}
\sum_{a=1}^K
\nicefrac{N_a(t)}{2\sigma^2}
\lVert\widetilde\mu_{t,a}-\nu_a\rVert_\infty^2,
\quad \inf\varnothing=+\infty.
\end{equation}
Thus, $Z_{t,i}$ measures the smallest change in the empirical
means, weighted by their sample counts, needed to reverse
the Pareto membership of arm $i$. It organically conveys the sensitivity of the current estimation of arm $i$; a larger value means more stable and thus should be more preferred.  We are now ready to define an adaptive temperature by reversely depending on the value of $Z_{t,i}$. 
For $c_\tau>0$, let us define
\begin{equation}\label{eq:soft}
\tau_{t,i}=\nicefrac{c_\tau}{\sqrt{1+Z_{t,i}}}.
\end{equation}
By default, we set 
$c_\tau=1$ unless explicitly stated otherwise. Then we find the soft target 
\(w_{t,i}=
\frac{\exp(-R_{t,i}/\tau_{t,i})}
{\sum_{j=1}^K\exp(-R_{t,j}/\tau_{t,j})}\) as a proxy for the original target $\frac{\mathbb I\{i\in\mathcal A^\star\}}{m}
$. 
As $Z_{t,i}$ increases, the temperature decreases, the weight for arm $i$ also drops. By doing so, it  serves as a penalty for the stability to incentivize exploring less stable arms. 

Finally, we derive the one-step fairness gain by 
$w_{t,i}\log(1+1/N_i(t))$. Since the normalization in
$w_{t,i}$ is common to all arms, the same ranking follows from 
\begin{equation}\label{eq:gain}
G_{t,i}=
\exp\{
-\nicefrac{R_{t,i}\sqrt{1+Z_{t,i}}}{c_\tau}
\}
\log(1+\nicefrac1{N_i(t)}).
\end{equation}

Though greedy selection appears appealing, it can be less exploratory given the uncertainties that we face in multiple coordinates. A top-two sampling idea has been proposed \citep{russo2020simple, qin2026dual} to sharply cope with this by maintaining two plausibly good arms and randomly selecting from them. We customize this idea through the design of the criterion and the sampling probability. Precisely, let $b_t,c_t$ be the two distinct arms with largest and second largest gains (i.e., the top two maximizers of $G_{t,i}$), and let the empirical Pareto front be $\widehat\Astar_t=\Astar(\widetilde\mu_t)$.

There are two alternatives that we can leverage to decide how to select among the top two arms, which are typical hard parts of adopting top-two algorithms (but very fair enough as a price for their elegancy and interpretability). First, the leader is selected with a constant probability $p$, e.g., $0.70$, and the challenger otherwise. When $|\widehat\Astar_t|=1$, the leader is selected surely. 

Second, the probability adaptively responds to the difference in the two candidates' sample counts:
\begin{equation}\label{eq:beta}
p_t=\begin{cases}
1,&|\widehat\Astar_t|=1,\\
\clip_{[0.45,0.95]}[0.70+0.25\tanh\left(\dfrac{N_{c_t}(t)-N_{b_t}(t)}{1+0.15\{N_{b_t}(t)+N_{c_t}(t)\}}\right)],
&b_t,c_t\in\widehat\Astar_t,\\
p, e.g., 0.70,&\text{otherwise}.
\end{cases}
\end{equation}
These parameter choices, such as $0.7$, are  for analytical convenience; in practice, we can tune them as needed. Subsequently, we choose the leader arm $b_t$ with probability $p_t$, and the challenger arm $c_t$ otherwise. Notably, after reliable identification, the two largest gains correspond to the two least-sampled true Pareto arms, for which $R_{t,i}=0$ and the gains reduce to $\log(1+1/N_i(t))$. The leader, which has the smaller arm count, exhibits a probability at least $0.70>1/2$, thereby favoring exploration. The full procedure of the proposed method reads as  Algorithm~\ref{alg:ungated}.

\begin{algorithm}[!htbp]
\caption{Soft-target adaptive top-two sampling}\label{alg:ungated}
\small
\begin{algorithmic}[1]
\Require Horizon $T\ge 2K$, $c_\tau>0$, and a fixed choice of the rule for $p_t$
\State Pull every arm twice; initialize $N_i(2K)$ and $\widetilde\mu_{2K,i}$
\For{$t=2K+1,\ldots,T$}
  \State Compute $R_{t,i}$, $Z_{t,i}$, and $G_{t,i}$ using \eqref{eq:Delta}--\eqref{eq:gain}
  \State Find the two distinct arms $b_t,c_t$ with largest gains
  \State Determine $p_t$ using the chosen rule
  \State Select $A_t=b_t$ with probability $p_t$; otherwise select $A_t=c_t$
  \State Pull $A_t$ and update the counts and empirical means
\EndFor
\end{algorithmic}
\end{algorithm}

\section{Theoretical Guarantee}

To the best of our knowledge, progress on provably balanced allocation has been quite limited in multi-objective bandits, largely due to the lack of a coherent and well-defined measure of ``balance''. It is our understanding that no bounds on the fairness measure across Pareto optimal arms exist. Fortunately, the new framework developed herein allows us to take steps in this direction with the powerful handle of the mature theory on KL divergence and distributional convergence, as well as the new method. We proceed as follows: we first derive an upper bound for this method in Section~\ref{upper_bound} and then sketch an information-theoretic lower bound for any reasonable policy in Section~\ref{lower_bound}.

\subsection{Upper Bound}\label{upper_bound}

Traditionally, we prove upper bounds on performance measures in expectation, as commonly seen in the bandit literature. Therefore, we first bound the expected fairness in this section. Naturally, an upper bound on expected fairness also bounds the pseudo-fairness $\bar{\mathcal{F}}_T$, defined over the expected allocation, by the convexity of KL divergence. Beyond this, it also seems both appealing and necessary to study the variance of fairness. By definition, the KL divergence depends inversely on the empirical measure induced by a sequence of random arm selections and thus may vary substantially across trajectories. The variance thus provides a measure of how reliably the policy achieves fair allocation, as ultimately we have a limited number of trajectories in practice. 

For the analysis, we introduce a separation margin that measures the sensitivity of the Pareto membership. Specifically, let us define
\begin{equation}\label{eq:margins}
\eta_\mu^{\rm P}
=\min_{i\in\Astar}\min_{j\ne i}\max_{r\in[d]}
(\mu_i^{(r)}-\mu_j^{(r)}), \quad  
\eta_\mu^{\rm D} =\min_{i\notin\Astar}\max_{j\ne i}\min_{r\in[d]}
(\mu_j^{(r)}-\mu_i^{(r)}),
\end{equation}
where $\eta_\mu^{\rm D}=+\infty$ if all arms are Pareto optimal.
The first margin, $\eta_\mu^{\rm P}$, measures the smallest advantage that a Pareto optimal arm has over any other arm in at least one objective. The second margin, $\eta_\mu^{\rm D}$, measures how clearly we can establish that an arm is dominated. Since it suffices to find one arm that outperforms it in every objective, we consider the arm with the largest minimum advantage across objectives. Taking the minimum over all dominated arms then identifies the hardest boundary.

Let $\eta_\mu=\min\{\eta_\mu^{\rm P},\eta_\mu^{\rm D}\}$.
We assume $\eta_\mu>0$, which ensures that every Pareto optimal arm is strictly better than every other arm in at least one objective, while every dominated arm is strictly dominated by some arm in all objectives, in analogue to strict dominance \citep{kone2024bandit}. Consequently, Pareto membership is preserved under sufficiently small perturbations of the mean vectors. It can be arbitrarily small and is unknown; smaller ones mean more sensitivity and harder problems.

Throughout the following theorems, we consider that $\mathcal M$ is bounded, $\eta_\mu>0$, and
\begin{equation}\label{eq:large-temp}
c_\tau>
32\sqrt{2}\frac{B_{\mathcal M}^2}{\eta_\mu^2},
\qquad
B_{\mathcal M}=\max_{r\in[d]}(b_r-a_r).
\end{equation}
To state the bounds, let us define
\(
s=
1-\frac{32\sqrt{2}B_{\mathcal M}^2}
{c_\tau\eta_\mu^2}\in(0,1)
\)
and
\begin{equation}\label{eq:ungated-horizon}
T_0=
\lceil
4\max\{
[
128K(1+\frac{384\sigma^2}{s\eta_\mu^2})
]^{8/s},
\;
\frac{4Kd}{s}
(1+\frac{32\sigma^2}{\eta_\mu^2}),
\;
[
\frac{32(K-m)c_\tau^2\sigma^2}{\eta_\mu^2}
]^2
\}
\rceil.
\end{equation}
The quantity $s$ measures how far the temperature parameter
lies above its sufficient lower bound. The three terms in
$T_0$ account for exploration of all arms, accumulated
estimation errors, and pulls of dominated arms, respectively.

\begin{theorem}
\label{thm:ungated}
Under the preceding conditions, for every $T\ge T_0$, we have 
\begin{align}\label{eq:ungated-mean}
\bar{\mathcal F}_T
\le \E_\mu\F_T
&\le
\frac{8}{T}
\left[
(K-m)\left(
3+\frac{c_\tau^2\sigma^2}{\eta_\mu^2}\log(T+1)
\right)
+\sqrt{T_0}+16m
\right]
+\frac{512T_0}{T^2}\log\frac{T}{m}.
\end{align}
Moreover,
\begin{align}\label{eq:ungated-variance}
\Var_\mu(\F_T)
&\le
\frac{128}{T^2}
[
(K-m)(
3+\frac{c_\tau^2\sigma^2}{\eta_\mu^2}\log(T+1)
)
+\sqrt{T_0}
]^2 
+\frac{16384m^2}{T}
+\frac{512T_0}{T^2}\log^2\frac{T}{m}.
\end{align}
\end{theorem}

In other words, we have $\E_\mu\F_T = O(\frac{\log T}{T})$ and $\Var_\mu(\F_T) = O(\frac{1}{T})$. Here, $T_0$ and the coefficients in the bounds jointly reflect the difficulty of the problem instance through their dependence on $\eta_{\mu}$, $\sigma$, $B_{\mathcal{M}}$, $K$, and $m = |\mathcal{A}^{\star}|$. In particular, $T_0$ increases as the temperature approaches its lower bound, corresponding to $s\downarrow0$, which may result from a small $\eta_{\mu}$, a feature of a hard problem instance. Indeed, we can further decompose the bounds into contributions from pulls of dominated arms and imbalance among Pareto optimal arms, which more precisely identify the sources of unfairness.
 
\begin{theorem}[Arm pulls and imbalance]
\label{thm:ungated-moments}
Under the same conditions, recall that 
\(
B_T=\sum_{i\notin\Astar}N_i(T),
\Phi_T=
\sum_{i\in\Astar}
\left(N_i(T)-\frac{T-B_T}{m}\right)^2.
\)
For every $T\ge2K$,
\begin{equation}\label{eq:ungated-B-moments}
\E_\mu B_T
\le
\sqrt{\E_\mu B_T^2}
\le
(K-m)\left(
3+\frac{c_\tau^2\sigma^2}{\eta_\mu^2}\log(T+1)
\right)
+\sqrt{T_0}.
\end{equation}
The imbalance potential satisfies
\begin{equation}\label{eq:ungated-Phi-moments}
\E_\mu\Phi_T\le T+T_0,
\qquad
\E_\mu\Phi_T^2\le T^2(T+T_0).
\end{equation}
\end{theorem}

Notably, when $m=1$, $\Phi_T=0$ identically, which is reduced to the optimization of efficiency in spirit to minimizing the Pareto regret; the functional is different though. 

The proof sketch is as follows; we defer the complete proof to Appendix \ref{app:ungated-proofs}. We first show that the condition on the temperature $c_{\tau}$ ensures logarithmic
exploration of every arm with high probability.
Combining this with concentration of the empirical means and positive
Pareto separation, we bound the second moment of the number of rounds
that fail the required sampling or accuracy conditions by $T_0$.
On the remaining rounds, the empirical Pareto set is correct, and a
dominated arm can be selected only while its count is at most
$(c_\tau^2\sigma^2/\eta_\mu^2)\log(t+1)$, since its gain otherwise
falls below those of the Pareto optimal arms. This gives the bounds
on dominated-arm pulls. Meanwhile, the gains of Pareto optimal arms
reduce to $\log(1+1/N_i(t))$, and the selection rule simply favors smaller
counts. We use this to show that the conditional expected
increase of $\Phi_t$ is at most one on these rounds, yielding
$\E_\mu\Phi_T\le T+T_0$. Its second moment then follows from
$\Phi_T^2\le T^2\Phi_T$.
To translate these bounds into fairness guarantees, we show that,
for $T\ge T_0$, every Pareto optimal arm receives at least $T/(8m)$
pulls except on an event of probability at most $512T_0/T^2$. Then the exact decomposition yields 
$\F_T\le8B_T/T+64m\Phi_T/T^2$; otherwise, we use the fact that 
$\F_T\le\log(T/m)$ in Lemma \ref{lem:decomp}.
Taking first and second moments yields the expected unfairness
and variance bounds. Finally, convexity of KL divergence gives
$\bar{\mathcal F}_T\le\E_\mu\F_T$.

It is useful to compare our sufficient horizon $T_0$ with
the sample complexity 
for Pareto set identification in \citet{garivier2024sequential}.
The latter, which is a pure exploration problem, derives the optimal sample
complexity $T^\star(\mu)\log(1/\delta)$ as the error
probability $\delta$ vanishes. For Gaussian observations with covariance $\sigma^2 I_d$,
our positive separation condition implies
\(
T^\star(\mu)\le\frac{8K\sigma^2}{\eta_\mu^2},
\)
since perturbing every mean coordinate by less than
$\eta_\mu/2$ cannot change the Pareto set.
Thus, both $T^\star(\mu)$ and $T_0$ reflect the difficulty
of distinguishing Pareto membership through the separation
margin. However, $T_0$ also depends on the
temperature parameter $c_\tau$ and can be substantially
larger than $T^\star(\mu)$. This 
reflects the additional difficulty of balancing
allocation, the exploration-exploitation trade-off, and hypothetically, the
conservatism of the choice of $T_0$.

\subsubsection{Implications for Pareto regret}
As mentioned earlier, fairness, by its definition, already measures the efficiency of the selected arms. To see this more precisely, let us introduce Pareto regret, a widely used metric for efficiency, and then derive bounds on it. We define the Pareto suboptimality gap of arm $i$ \citep{drugan2013designing,guan2026stochastic} by
\(
\Delta_i^{\rm P}
=
[
\max_{j\ne i}\min_{r\in[d]}
\bigl(\mu_j^{(r)}-\mu_i^{(r)}\bigr)
]_+,
\)
and let $\Delta_{\max}^{\rm P}=\max_{i\in[K]}\Delta_i^{\rm P}$.
The cumulative Pareto regret is
\[
R_T^{\rm P}
=
\sum_{t=1}^T\Delta_{A_t}^{\rm P}
=
\sum_{i\notin\Astar}\Delta_i^{\rm P}N_i(T).
\]
Under the condition of positive Pareto
separation, we immediately obtain $\Delta_i^{\rm P}\ge\eta_\mu > 0$ for every
dominated arm.

\begin{corollary}[Pareto regret]
\label{cor:ungated-pareto-regret}
Under the assumptions of Theorem~\ref{thm:ungated-moments},
for every $T\ge2K$,
\begin{align}\label{eq:ungated-pareto-regret}
\E_\mu R_T^{\rm P}
&\le
c_\tau^2\sigma^2\log(T+1)
\sum_{i\notin\Astar}\frac1{\Delta_i^{\rm P}}
+
3\sum_{i\notin\Astar}\Delta_i^{\rm P}
+
\Delta_{\max}^{\rm P}\sqrt{T_0},
\end{align}
where $T_0$ is the same explicit quantity used in
Theorem~\ref{thm:ungated-moments}.
Moreover, on every trajectory,
\begin{equation}\label{eq:regret-fairness}
R_T^{\rm P}
\le
\Delta_{\max}^{\rm P}B_T
\le
\Delta_{\max}^{\rm P}T\F_T.
\end{equation}
\end{corollary}

Consequently, on the positive side, the same policy achieves logarithmic expected
Pareto regret, consistent with existing work on Pareto regret
\citep{drugan2013designing, guan2026stochastic}, while additionally approaching a balanced allocation over
the Pareto set. In terms of rates, it achieves the best of both worlds for these two goals. Importantly, we are not in a position to make the same claim about the coefficients (and thus the full term of the bound) without a formal comparison. In fact, on the flip side, this advantage comes at a cost; the Pareto regret has an additional term involving $\sqrt{T_0}$ and a  potentially larger coefficient of the leading $\log T$ term through the temperature parameter $c_{\tau}$.

\subsection{Information-theoretic Lower Bounds}\label{lower_bound}

Whenever we propose a new metric, we are interested in establishing a lower bound on this metric for any reasonable policy. In other words, for any such policy and any problem instance, the expected fairness is at least a certain value. To do so, for a problem instance, we construct an alternative instance that are statistically difficult to distinguish, but on which similar sampling behavior can lead to substantially different fairness guarantees. We then show that the average expected fairness over these two instances must be at least a certain value, so that at least one of them has to attain this value.

In our context, this derivation has a dual purpose. On the one hand, it reveals the gap between the established upper bound and this unavoidable cost, thereby assessing the sharpness of the method. Under positive Pareto separation, we show that the method is rate-optimal, up to differences in the coefficients. On the other hand, without this condition, we can sense how far a method can go. In fact, we show that the absence of positive separation leads to a lower bound with a larger rate. As a by-product, this suggests that the condition is necessary for the rate. More deeply, it precisely manifests the sensitivity of the Pareto front, a genuinely hard aspect of multi-objective problems.

\subsubsection{Positive Pareto separation}

We first prove a lower bound on the fairness achievable
by any reasonable policy under positive Pareto separation.
The key observation is that a dominated arm may become
Pareto optimal under an alternative instance, and
distinguishing the two instances yields   exploration cost.
We characterize this cost through the KL divergence between
their reward distributions.

For an instance with mean vectors $\mu$, let $P_i$ denote
the reward distribution of arm $i$, and as before we denote $\E_\mu$
as the expectation under these distributions. We focus on policies that are consistent across instances with respect to fairness (as opposed to regret \citep{lai1985asymptotically, kaufmann2016complexity, lattimore2019bandit}), namely uniformly fair below. 
\begin{definition}
\label{def:efficiency}
A policy is uniformly fair if, for every
admissible collection of reward distributions with mean
vectors $\mu$ satisfying $\eta_\mu>0$ and every
$\zeta\in(0,1)$,
\[
\E_\mu\F_T=o(T^{-1+\zeta}).
\]
\end{definition}

Intuitively, it means that the policy should be sublinear in $T$ w.r.t. $T\E_\mu\F_T$, in analogue to sublinear regret \citep{kaufmann2016complexity, lattimore2019bandit}. For a dominated arm $i$, consider replacing its reward
distribution $P_i$ with an admissible distribution $Q$; admissibility means the same sub-Gaussian
assumption in $\mathcal M$.
Let $\nu$ denote the resulting mean vectors, so that
$\nu_i=\E_{X\sim Q}X$ and $\nu_a=\mu_a$ for $a\ne i$. 
Let us define
\[
\kappa_i^{\rm KL}(\mu)
=
\inf_{\substack{
Q:\,\eta_\nu>0\\
i\in\Astar(\nu)
}}
\KL(P_i\|Q),
\]
with $\inf\varnothing=+\infty$ and $1/\infty=0$.
This quantity measures, among all alternatives $\nu$ under which arm $i$ becomes Pareto optimal, the smallest KL divergence to $\nu$. It thus measures the statistical difficulty of determining the Pareto membership of arm $i$; smaller values indicate that the instances are harder to distinguish. Consequently, we are ready to derive the following instance-dependent lower bound.  

\begin{theorem}
\label{thm:lower-asymptotic}
For every instance with $\eta_\mu>0$ and every
uniformly fair policy,
\begin{equation}\label{eq:lower-asymptotic}
\liminf_{T\to\infty}
\frac{T\E_\mu\F_T}{\log T}
\ge
\sum_{i\notin\Astar(\mu)}
\frac{1}{\kappa_i^{\rm KL}(\mu)}.
\end{equation}
\end{theorem}

The lower bound increases as $\kappa_i^{\rm KL}(\mu)$ decreases, reflecting the greater exploration required to distinguish the Pareto membership of arm $i$.
Whenever the right-hand side of
\eqref{eq:lower-asymptotic} is positive, the expected
fairness converge with a $\log T/T$ rate 
within this policy class. This lower bound matches
the rate of our upper bound, and we hope to inspire future work that improves the 
coefficients.

A careful reader may observe that the bound is asymptotic. Noting that the upper bound holds for a finite horizon ($\geq T_0$), it may be favorable to also prove a finite-time lower bound. For this
purpose,

\begin{theorem}[Finite-time fairness lower bound]
\label{thm:lower-finite}
Fix an instance $\mu$ with $\eta_\mu>0$ and a nonempty
subset $\mathcal I\subseteq[K]\setminus\Astar(\mu)$.
For each $i\in\mathcal I$, we choose an admissible reward distribution $Q_i$ and let $\nu^{(i)}$ be the instance obtained by replacing only 
$P_i$ with $Q_i$, s.t. $\eta_{\nu^{(i)}} >0$, $i \in \mathcal A^{\star}$, and $0 < \KL(P_i\|Q_i)< \infty$. 
Then, for any policy and every
$T\ge\max\{K,4eK(K+1)S\}$ where \(
S=\sum_{i\in\mathcal I}\nicefrac1{\KL(P_i\|Q_i)}
\),
\begin{equation}\label{eq:lower-fair-finite}
\max\left\{
\E_\mu\F_T,\,
\max_{i\in\mathcal I}\E_{\nu^{(i)}}\F_T
\right\}
\ge
\frac{S}{T}
\left(
\log\frac{T}{4K(K+1)S}
-
\log\log\frac{T}{4K(K+1)S}
\right).
\end{equation}
\end{theorem}

In particular, the worst-case expected unfairness over
these instances is at least
$(1-o(1))S\log T/T$, uniformly over policies. The right-hand side is asymptotic to
\(
\nicefrac{\log T}{T}
\sum_{i\in\mathcal I}\frac{1}{\KL(P_i\|Q_i)}
\); when $\mathcal I^{c} = \mathcal{A}^{\star}(\mu)$, this is in spirit with the asymptotic bound above. Intuitively, statistically closer instances may require longer
horizons before the asymptotic lower bound becomes informative. The constants can be consistent with Theorem~\ref{thm:ungated}, thus supporting the rate optimality of the method for sufficiently large $T$ whenever $\kappa_{i}^{\text{KL}} < \infty$. Taking a closer look at the coefficient, we observe that
\[
\sum_{i\in\mathcal I}
\frac{1}{
\KL(P_i\|Q_i)
}
\leq
\sum_{i\notin\Astar(\mu)}
\frac{1}{\kappa_i^{\rm KL}(\mu)}.
\]
Thus, the finite-time bound is more conservative, as a price we pay for a nonasymptotic guarantee. In a like manner, this conservatism appears in finite-time lower
bounds on regret for classical bandits \citep{lattimore2019bandit}.

\subsubsection{Zero Pareto separation}
\label{sec:boundary}

When the separation is zero, it becomes a different story. Thereon, an arbitrarily small change
in the mean rewards may change the Pareto set and,
consequently, the ideal distribution. We show that this
instability prevents any reasonable policy from achieving the
aforementioned rate both at such a boundary instance.  

To see this, consider two arms with two objectives and
independent Gaussian noise $\mathcal N(0,\sigma^2 I_2)$,
where $\sigma>0$ and $g\in(0,1)$ are fixed:
\begin{equation}\label{eq:boundary-instances}
\begin{aligned}
\nu_0:&\quad
\mu_1=(0,0),\qquad \mu_2=(g,0),\\
\nu_\varepsilon:&\quad
\mu_1=(0,\varepsilon),\qquad \mu_2=(g,0),
\quad 0<\varepsilon<1.
\end{aligned}
\end{equation}
Under $\nu_0$, arm 2 dominates arm 1, but
$\eta_{\nu_0}=0$ because their second coordinates are equal.
Under $\nu_\varepsilon$, both arms are Pareto optimal, with
$\eta_{\nu_\varepsilon}=\min\{g,\varepsilon\}>0$.
Thus, the target distribution across arms changes from $(0,1)$ to
$(1/2,1/2)$, regardless of how small $\varepsilon$ is. That is, sampling arm 1 is necessary for fairness
under $\nu_\varepsilon$, but accumulates unfairness under $\nu_0$.
This dilemma diverts a policy from achieving the same fairness rate
at $\nu_0$ and $\nu_\varepsilon$.
The formal statement is as follows.

\begin{theorem}
\label{thm:boundary-superlog}
For any policy that is uniformly fair under positive Pareto separation,
the instances in \eqref{eq:boundary-instances} satisfy
\begin{equation}\label{eq:boundary-superlog}
\frac{\E_{\nu_0}N_1(T)}{\log T}\longrightarrow\infty,
\qquad
\frac{T\E_{\nu_0}\F_T}{\log T}\longrightarrow\infty.
\end{equation}
\end{theorem}

We observe a subtle difficulty in fairness learning at the boundary of Pareto optimality.  An arbitrarily small
change in the mean vectors ($\epsilon$) can alter the Pareto set. Statistically distinguishing these
instances requires increasingly many observations as the separation $\eta_{\nu_\varepsilon}$ 
approaches zero, while failing to do so can lead to unfair allocations of rate $\omega(\nicefrac{\log{T}}{T})$. 

\section{Variance Reduction}\label{sec:safeguard}

To quantify the relative variability of fairness across trajectories, let us adopt the signal-to-noise ratio (SNR), commonly defined as
\[
\mathrm{SNR}_T
=
\nicefrac{\mathbb E[\mathcal F_T]}
{\sqrt{\operatorname{Var}(\mathcal F_T)}}.
\]
The aforementioned upper and lower bounds imply that
$\mathbb E[\mathcal F_T]=\Theta(\log T/T)$ and
$\operatorname{Var}(\mathcal F_T)=O(1/T)$. The standard deviation decreases more slowly
than the expected fairness. Consequently, we obtain
$\mathrm{SNR}_T=\Omega(\log T/\sqrt T)$.
This has direct implications for empirical evaluation.
Precisely, given $R$ independent runs (i.e., trajectories),
we compute the sample mean of fairness $\overline{\mathcal F}_{T,R}$.
Then its relative standard error is
\[
\nicefrac{\sqrt{\operatorname{Var}(\overline{\mathcal F}_{T,R})}}
{\mathbb E[\mathcal F_T]}
=
\nicefrac{1}{\sqrt R\,\mathrm{SNR}_T}
=
O\left(\nicefrac{1}{\log T}\sqrt{\nicefrac{T}{R}}\right).
\]
This means that the evaluation is computationally expensive and thus less desirable, as maintaining a fixed relative estimation precision
may require $R=\Omega(T/\log^2 T)$ in the worst-case scenario where
$\operatorname{Var}(\mathcal F_T)=\Theta(1/T)$.
Even when expected fairness vanishes at the rate
$\Theta(\log T/T)$, reliable estimation of this expectation by
the sample mean $\overline{\mathcal F}_{T,R}$
may require a number of runs that grows nearly linearly
with the horizon. To this end, we hope to propose a new method
for variance reduction, so that we can increase
$\mathrm{SNR}_T$ and decrease $R$.

\subsection{Methodology}

The original top-two rule selects arms according to their
empirical gains in \eqref{eq:gain}, which are in the direction of reducing empirical unfairness. A true Pareto optimal arm whose mean is
underestimated may receive few further observations,
delaying the correction of its estimate and causing persistent
underallocation. Indeed, such trajectories likely exhibit extreme unfairness and thus can contribute substantially
to the variance of fairness. To fix this, we introduce two changes. We first promote extra exploration based
directly on arm counts in the early stages, and add a confidence test for terminating this exploration timely. 
The complete procedure is summarized in
Algorithm~\ref{alg:egatt}. 

For a horizon $T>2K$, we set $q_0=96$ and $p_g=0.70$,
and define the exploration quota
\begin{align}\label{eq:quota}
\rho_T
&=[
\sqrt T-q_0\log(1+\nicefrac{T}{2K})
]_+,\\
h_{t,T}
&=2+q_0\log(1+\nicefrac{t}{2K})
+\rho_T\min\{1,\nicefrac{t}{\lceil\log T\rceil}\}.
\notag
\end{align}
The logarithmic term in $h_{t,T}$ provides continuing
exploration, while $\rho_T \sim\sqrt T$ allows additional exploration
when the separation is too small for this term alone
to discover Pareto membership.  The factor multiplying $\rho_T$
raises the quota early in the horizon, ideally before
imbalances accumulate. 

While the exploration is active, let
$j_t\in\arg\min_iN_i(t)$ be a least-sampled arm. If $N_{j_t}(t)<h_{t,T}$,
we select $j_t$ with probability $p_g$; with 
probability $1- p_g$, we select the leader $b_t$ or the challenger
$c_t$ if $j_t$ is already the leader $b_t$. Otherwise, we use the original top-two rule. The main difference is that an arm receives this protection
regardless of its gain or estimated Pareto membership.

While this directly addresses the prolonged underallocation, we need to stop it  once sufficient statistical evidence is available, to prevent imbalance. The key is to use the empirical separation
relative to the uncertainty in the estimates. 
Let
\(
\Lambda_T=\log(4KdT^7), 
n_{\min}(t)=\min_iN_i(t),
\)
and we define the empirical separation and uncertainty as 
\begin{align}\label{eq:gate}
\widehat\eta_t^{\rm P}
&=\min_{i\in\widehat\Astar_t}\min_{j\ne i}
\max_{r\in[d]}
(\widetilde\mu_{t,i}^{(r)}-\widetilde\mu_{t,j}^{(r)}),\\
\widehat\eta_t^{\rm D}
&=\min_{i\notin\widehat\Astar_t}\max_{j\ne i}
\min_{r\in[d]}
(\widetilde\mu_{t,j}^{(r)}-\widetilde\mu_{t,i}^{(r)}),
\notag\\
\widehat\eta_t
&=\min\{\widehat\eta_t^{\rm P},\widehat\eta_t^{\rm D}\},
\qquad
r_t^{\max}
=\sqrt{\nicefrac{2\sigma^2\Lambda_T}{n_{\min}(t)}}.
\notag
\end{align}
Here, $\widehat\eta_t^{\rm D}=+\infty$ when all arms
are empirically Pareto optimal. These are the empirical
counterparts of the separation margins,
and $r_t^{\max}$ bounds the uncertainty by leveraging the smallest
arm count. The choice of $\Lambda_T$ ensures that the 
confidence bounds hold simultaneously through coordinates and time $T$
with probability at least $1-1/(2T^6)$.

We permanently disable the exploration after 
\(
\widehat\eta_t\ge8r_t^{\max}.
\)
The empirical comparisons are then sufficiently separated
relative to their uncertainty to assign Pareto membership
on the confidence event. Thereafter, we follow the
original top-two rule to promote balanced allocation as before.

\begin{algorithm}[t]
\caption{Evidence-gated adaptive top-two sampling (EG-ATT)}
\label{alg:egatt}
\begin{algorithmic}[1]
\Require Horizon $T>2K$, proxy $\sigma$,
$q_0=96$, $p_g=0.70$, $c_\tau=1$
\State Pull each arm twice;
set $\mathsf{resolved}\gets\mathrm{false}$
\State Compute $\rho_T$ from \eqref{eq:quota}
and set $\Lambda_T=\log(4KdT^7)$
\For{$t=2K+1,\ldots,T$}
    \State Compute $\widetilde\mu_t$, $\widehat\Astar_t$,
    and the gains $G_{t,i}$ from \eqref{eq:gain}
    \State Identify the two distinct arms $b_t,c_t$
    with the largest gains, using fixed tie-breaking
    \State Compute $p_t$ from \eqref{eq:beta},
    $h_{t,T}$ from \eqref{eq:quota}, and
    $n_{\min}(t)=\min_iN_i(t)$
    \If{$\mathsf{resolved}=\mathrm{false}$}
        \State Compute $\widehat\eta_t$ and $r_t^{\max}$
        from \eqref{eq:gate}
        \If{$\widehat\eta_t\ge8r_t^{\max}$}
            \State $\mathsf{resolved}\gets\mathrm{true}$
        \EndIf
    \EndIf
    \If{$\mathsf{resolved}=\mathrm{false}$
    \textbf{and} $n_{\min}(t)<h_{t,T}$}
        \State Let $j_t\in\arg\min_iN_i(t)$,
        breaking ties deterministically
        \State $k_t\gets b_t$ if $j_t\ne b_t$,
        else $k_t\gets c_t$
        \State Draw $A_{t}=j_t$ with probability $p_g$,
        else $A_{t}=k_t$
    \Else
        \State Draw $A_{t}=b_t$ with probability $p_t$,
        else $A_{t}=c_t$
    \EndIf
    \State Pull $A_{t}$ and update the counts
    and empirical means
\EndFor
\end{algorithmic}
\end{algorithm}

\subsection{Theoretical Guarantee}
We show that the new mechanism controls the
number of pulls of dominated arms, while the adaptive
top-two rule controls the imbalance among Pareto optimal
arms after release. These properties yield
bounds on the expected fairness, which should remain of
the same order as before since we use the same imbalance
control, and on its variance, which decreases due to
the additional control.

\begin{theorem}
Suppose $\eta_\mu>0$ and consider Algorithm~\ref{alg:egatt}.
Define the constant
\(
C_\mu
=
2K+5+\nicefrac2{p_g}
[K(1+\nicefrac{512\sigma^2}{\eta_\mu^2})+8].
\)
For every horizon $T>2K$ satisfying \(\rho_T\ge
\max\{2,\lceil
\nicefrac{512\sigma^2\Lambda_T}{\eta_\mu^2}
\rceil\}\) and \(
T\ge(1+8m)C_\mu\Lambda_T,\) 
there exists a constant $D>0$, depending only on $|A^{\star}|$,
such that
\begin{align}
\E_\mu\F_T \le
\frac{2C_\mu\Lambda_T}{T}
+\frac{D}{T^2}, \quad 
\Var_\mu(\F_T)\le
\frac{8C_\mu^2\Lambda_T^2}{T^2}
+\frac{D\log^2 T}{T^4}.
\end{align}
Moreover, we can bound the moments of the arm pulls and imbalance potential as  
\(
\E_\mu B_T^4
\le C_\mu^4\Lambda_T^4+T^{-2}\) and \(
\E_\mu \Phi_T^2\le D+T^{-2}. 
\)
\end{theorem}

In terms of order, we obtain
\(
\E_\mu\F_T
=O_\mu(\frac{\log T}{T})\) and
\(
\Var_\mu(\F_T)
=O_\mu(\frac{\log^2T}{T^2})
\) by noting that $\Lambda_T = O(\log T)$.
On the positive side, compared to Theorem~\ref{thm:ungated}, we preserve the same order
of expected fairness but substantially reduce the order of the variance,
from $O(\frac{1}{T})$ to $O_\mu(\frac{\log^2T}{T^2})$.
This quadratic decay has a clear practical advantage.
To see this, the relative standard error becomes
\[
\nicefrac{\sqrt{\operatorname{Var}(\overline{\mathcal F}_{T,R})}}
{\mathbb E[\mathcal F_T]}
=
O(R^{-1/2}); 
\]
it leverages the fact that the upper bound on the expected fairness implies that the policy is uniformly fair (sublinearity holds by the logarithmic rate). Therefore, maintaining a fixed relative estimation precision
may only require a constant number of runs, independent of
the time horizon $T$, in contrast to the previous order of
$T/\log^2 T$ runs.

However, on the flip side, the upper bounds 
differ in their leading coefficients. The consequences, and thus the decision regarding which method to use, depend on the interplay between the separation margin $\eta_{\mu}$ and the horizon $T$. Precisely, for the mean fairness, with $m=|\mathcal A^\star|$, the coefficient
of $\log T/T$ is
$8(K-m)c_\tau^2\sigma^2/\eta_\mu^2$ for the original top-two rule where $c_\tau>32\sqrt{2}B_{\mathcal M}^2/\eta_\mu^2$,
compared with
$14C_\mu=68K+390+20480K\sigma^2/\eta_\mu^2$
for the new variant. We observe that when $\eta_{\mu}$ is really small, which means that it is hard for us to sort out Pareto membership, using the new variant may offer a greater advantage, as its coefficient is of order $\nicefrac{1}{\eta_{\mu}^2}$, whereas the original rule exhibits $\nicefrac{1}{\eta_{\mu}^6}$. This appears very natural and consistent with our earlier motivation. When $\eta_{\mu}$ is large, then the plausibly larger coefficient of the new variant reflects
the trade-off introduced by the variance reduction mechanism. The quota allocates additional pulls to undersampled arms, including dominated arms, and thus is more conservative, at the cost of balancing the
Pareto optimal arms. When the horizon $T$ is very large (e.g., at least $10^5$), it is still  surely more advantageous to use the new variant for its rate advantage. At smaller horizons, such an advantage fades away as the larger coefficient unfortunately becomes the bottleneck. Perhaps when the problem is easier, it is better to be less conservative, as the original method does. In fact, our previous concern about $R=\Omega(T)$ may also fade, and then the original top-two method remains a safer choice.

For the variance, the comparison is more subtle, as the original
top-two rule has both a $1/T$ term and terms of order
$\log^2 T/T^2$ in its upper bound. The coefficient of the
former is $16384m^2$, independent of $\eta_\mu$, whereas
the coefficient of the latter includes
$128(K-m)^2c_\tau^4\sigma^4/\eta_\mu^4$,
which grows at the order of $1/\eta_\mu^{12}$ when
$c_\tau\asymp 1/\eta_\mu^2$.
In comparison, the new variant has a variance bound of order
$\log^2 T/T^2$, with coefficient
$392C_\mu^2=2(68K+390+20480K\sigma^2/\eta_\mu^2)^2$,
which grows at the order of $1/\eta_\mu^4$.
Therefore, when $\eta_\mu$ is small, the new variant may
offer an advantage not only through its faster variance
decay but also through its dependence on the separation
margin. When $\eta_\mu$ is larger, the advantage still holds when the horizon is large enough; nevertheless, it likewise becomes less clear at smaller horizons. 

\section{Simulation Experiments}\label{sec:fairness-numerics}

Conventionally, upper bounds are conservative and thus reflect performance
in a pessimistic sense. A pathway to examine potentially more optimistic
performance is to conduct numerical studies and report the observed
performance. Oftentimes, simulation experiments suffice for this purpose. We explore this idea in the context of fairness herein.
Another advantage is that we can sense how the separation margin
and the time horizon affect the actual mean and variance of unfairness
under the two methods. Moreover, this supports comparisons with
existing methods for stochastic multi-objective bandits. Though these
methods are mostly empirical or heuristic, they sometimes do work in practice;  we believe that it is
crucial to understand when they work, when they do not, which ones
work, and why.

\paragraph{Experimental setup.}
We consider a six-arm Bernoulli instance with four Pareto-optimal arms
and a five-arm Gaussian instance with three Pareto-optimal arms.
The separation margins are $\eta_\mu\in\{0.005,0.01,0.1\}$ for Bernoulli
rewards and $\eta_\mu\in\{0.005,0.02,0.1\}$ for Gaussian rewards.
Within each family, we vary the margin through affine transformations
of the mean reward vectors while preserving the Pareto set.
All mean vectors lie in $[0,1]^2$, and the Gaussian rewards have
standard deviation $\sigma=0.12$.
The mean vectors are reported in Appendix~\ref{app:fairness-instance-details}.

We consider $T\in\{5{,}000,50{,}000\}$ and initialize every arm
with two observations. For each instance, we measure terminal unfairness
$\mathcal F_T=D_{\mathrm{KL}}(\pi^\star\|\widehat\pi_T)$,
where $\pi^\star=\operatorname{Unif}(\mathcal A^\star)$ and
$\widehat\pi_T(i)=N_i(T)/T$. We report its sample mean and variance
over $240$ independent held-out runs.

The baselines include Pareto-UCB1~\citep{drugan2013designing},
Pareto Thompson sampling (PTS)~\citep{yahyaa2014annealing},
Width-guided UCB~\citep{guan2026stochastic},
Pareto-KG~\citep{yahyaa2014knowledge}, and
Annealing-Pareto~\citep{yahyaa2014annealing}. Their parameters mostly follow from what have been reported in their respective work. We tune ATT and EG-ATT separately for each instance and horizon through
noisy Bayesian optimization (BO), using two objectives:
\begin{align*}
J_{\mathrm{mean+SD}}
=\widehat{\mathbb E}\mathcal F_T
+\frac12\sqrt{\widehat{\operatorname{Var}}(\mathcal F_T)}, 
J_{\log}=\frac12\log\!\left(\widehat{\mathbb E}\mathcal F_T+10^{-4}\right)+\frac12\log\!\left(\widehat{\operatorname{Var}}(\mathcal F_T)+10^{-6}\right).
\end{align*}
We vary the temperature and leader probability of both methods,
and additionally the quota and gate parameters of EG-ATT.
Each search evaluates $20$ configurations on $28$ pilot runs,
compares the leading candidates on $90$ validation runs,
and evaluates on the held-out runs.

Tables~\ref{tab:fairness_bernoulli_two_T}
--\ref{tab:fairness_gaussian_two_T} summarize the results. We hope to interpret them and extract the insights in what follows. 

\begin{table*}[t]
\centering
\small
\setlength{\tabcolsep}{3.2pt}
\caption{Mean and variance of terminal KL unfairness on Bernoulli rewards for $T\in\{5{,}000,50{,}000\}$. The smallest mean and variance for each separation margin and horizon are in bold.}
\label{tab:fairness_bernoulli_two_T}
\begin{tabular}{lrrrrrr}
\toprule
 & \multicolumn{2}{c}{$\eta_\mu=0.005$} & \multicolumn{2}{c}{$\eta_\mu=0.01$} & \multicolumn{2}{c}{$\eta_\mu=0.1$} \\
\cmidrule(lr){2-3}\cmidrule(lr){4-5}\cmidrule(lr){6-7}
Method & Mean & Variance & Mean & Variance & Mean & Variance \\
\midrule
\multicolumn{7}{l}{\textit{Horizon $T=5,000$}} \\
ATT (mean+SD BO) & $0.3655$ & $0.0042$ & $0.3041$ & $0.0090$ & $0.0246$ & $0.0028$ \\
ATT (log BO) & $\boldsymbol{0.3651}$ & $0.0043$ & $0.3099$ & $0.0048$ & $0.0246$ & $0.0028$ \\
EG-ATT (mean+SD BO) & $0.3719$ & $0.0014$ & $\boldsymbol{0.3020}$ & $0.0100$ & $0.0227$ & $6.55\times 10^{-5}$ \\
EG-ATT (log BO) & $0.4048$ & $\boldsymbol{2.70\times 10^{-6}}$ & $0.4053$ & $\boldsymbol{7.50\times 10^{-10}}$ & $0.2650$ & $\boldsymbol{1.92\times 10^{-14}}$ \\
Pareto-UCB1 & $0.3686$ & $0.0015$ & $0.3291$ & $0.0015$ & $0.1116$ & $0.00034$ \\
PTS & $0.4355$ & $0.0319$ & $0.3418$ & $0.0259$ & $\boldsymbol{0.0222}$ & $0.0025$ \\
Width-guided UCB & $0.3753$ & $0.0016$ & $0.3345$ & $0.0020$ & $0.2208$ & $0.0067$ \\
Pareto-KG (adapted) & $0.6248$ & $0.2845$ & $0.4734$ & $0.2805$ & $0.3119$ & $0.2331$ \\
Annealing-Pareto & $3.0503$ & $0.9311$ & $3.0684$ & $0.8474$ & $2.0157$ & $0.6267$ \\
\midrule
\multicolumn{7}{l}{\textit{Horizon $T=50,000$}} \\
ATT (mean+SD BO) & $0.2495$ & $0.0168$ & $\boldsymbol{0.1368}$ & $0.0282$ & $0.0068$ & $0.0057$ \\
ATT (log BO) & $0.2758$ & $0.0044$ & $0.1578$ & $0.0044$ & $0.0140$ & $0.0182$ \\
EG-ATT (mean+SD BO) & $\boldsymbol{0.2488}$ & $0.0200$ & $0.1415$ & $0.0066$ & $\boldsymbol{0.0028}$ & $7.72\times 10^{-7}$ \\
EG-ATT (log BO) & $0.2730$ & $0.0047$ & $0.1566$ & $0.0044$ & $0.0036$ & $\boldsymbol{2.08\times 10^{-7}}$ \\
Pareto-UCB1 & $0.3079$ & $\boldsymbol{0.0012}$ & $0.2460$ & $\boldsymbol{0.0012}$ & $0.0306$ & $1.64\times 10^{-5}$ \\
PTS & $0.2846$ & $0.0276$ & $0.1837$ & $0.0272$ & $0.0033$ & $2.68\times 10^{-6}$ \\
Width-guided UCB & $0.2940$ & $0.0021$ & $0.2032$ & $0.0012$ & $1.6001$ & $0.0106$ \\
Pareto-KG (adapted) & $0.2556$ & $0.0628$ & $0.1436$ & $0.0457$ & $0.2565$ & $0.3260$ \\
Annealing-Pareto & $4.7997$ & $1.5698$ & $4.5389$ & $1.7251$ & $2.8636$ & $1.3302$ \\
\bottomrule
\end{tabular}
\end{table*}

\begin{table*}[t]
\centering
\small
\setlength{\tabcolsep}{3.2pt}
\caption{Mean and variance of terminal KL unfairness on Gaussian rewards for $T\in\{5{,}000,50{,}000\}$. The smallest mean and variance for each separation margin and horizon are in bold.}
\label{tab:fairness_gaussian_two_T}
\begin{tabular}{lrrrrrr}
\toprule
 & \multicolumn{2}{c}{$\eta_\mu=0.005$} & \multicolumn{2}{c}{$\eta_\mu=0.02$} & \multicolumn{2}{c}{$\eta_\mu=0.1$} \\
\cmidrule(lr){2-3}\cmidrule(lr){4-5}\cmidrule(lr){6-7}
Method & Mean & Variance & Mean & Variance & Mean & Variance \\
\midrule
\multicolumn{7}{l}{\textit{Horizon $T=5,000$}} \\
ATT (mean+SD BO) & $0.3188$ & $0.0098$ & $0.0516$ & $0.0123$ & $\boldsymbol{0.0018}$ & $1.94\times 10^{-6}$ \\
ATT (log BO) & $0.3468$ & $0.0078$ & $0.0516$ & $0.0123$ & $0.0026$ & $1.12\times 10^{-6}$ \\
EG-ATT (mean+SD BO) & $\boldsymbol{0.2826}$ & $0.0196$ & $\boldsymbol{0.0390}$ & $0.0020$ & $0.0072$ & $\boldsymbol{<10^{-15}}$ \\
EG-ATT (log BO) & $0.4923$ & $\boldsymbol{4.89\times 10^{-5}}$ & $0.2196$ & $\boldsymbol{6.07\times 10^{-15}}$ & $0.0072$ & $2.19\times 10^{-9}$ \\
Pareto-UCB1 & $0.4698$ & $0.00024$ & $0.3860$ & $0.00027$ & $0.1462$ & $3.90\times 10^{-5}$ \\
PTS & $0.3175$ & $0.0321$ & $0.0513$ & $0.0081$ & $0.0035$ & $1.75\times 10^{-6}$ \\
Width-guided UCB & $0.4859$ & $0.00023$ & $0.4351$ & $0.00023$ & $0.3289$ & $0.00014$ \\
Pareto-KG (adapted) & $0.3118$ & $0.0243$ & $0.0757$ & $0.0086$ & $0.0176$ & $0.0255$ \\
Annealing-Pareto & $0.3928$ & $0.0380$ & $0.1352$ & $0.0147$ & $0.0443$ & $9.72\times 10^{-6}$ \\
\midrule
\multicolumn{7}{l}{\textit{Horizon $T=50,000$}} \\
ATT (mean+SD BO) & $0.0713$ & $0.0168$ & $0.0066$ & $7.44\times 10^{-6}$ & $0.00030$ & $1.47\times 10^{-8}$ \\
ATT (log BO) & $0.0946$ & $0.0017$ & $0.0066$ & $7.44\times 10^{-6}$ & $\boldsymbol{0.00018}$ & $2.03\times 10^{-7}$ \\
EG-ATT (mean+SD BO) & $0.0699$ & $0.0046$ & $\boldsymbol{0.0046}$ & $3.22\times 10^{-6}$ & $0.0017$ & $\boldsymbol{<10^{-15}}$ \\
EG-ATT (log BO) & $\boldsymbol{0.0603}$ & $0.0055$ & $0.0069$ & $\boldsymbol{1.34\times 10^{-6}}$ & $0.0013$ & $6.68\times 10^{-12}$ \\
Pareto-UCB1 & $0.4194$ & $0.00020$ & $0.2375$ & $8.11\times 10^{-5}$ & $0.0382$ & $1.95\times 10^{-6}$ \\
PTS & $0.0919$ & $0.0273$ & $0.0070$ & $8.60\times 10^{-6}$ & $0.00052$ & $3.50\times 10^{-8}$ \\
Width-guided UCB & $0.4538$ & $\boldsymbol{0.00015}$ & $0.3667$ & $0.00011$ & $0.7097$ & $0.00027$ \\
Pareto-KG (adapted) & $0.1254$ & $0.0060$ & $0.0139$ & $0.0071$ & $0.0099$ & $0.0216$ \\
Annealing-Pareto & $0.8993$ & $0.2025$ & $0.1864$ & $0.1639$ & $0.0044$ & $8.48\times 10^{-8}$ \\
\bottomrule
\end{tabular}
\end{table*}

\paragraph{Effect of the separation margin and horizon.} A larger separation margin reduces mean unfairness for both proposed methods in every instance considered. For example, at $T=50{,}000$, the mean unfairness of ATT on Gaussian rewards decreases from $0.0713$ at $\eta_\mu=0.005$ to $2.96\times10^{-4}$ at $\eta_\mu=0.1$, while that of EG-ATT decreases from $0.0699$ to $0.00168$. This coincides with the theory, where the bounds scale reversely with the margin $\eta$ and also the intuition that a smaller margin may indicate a harder instance. Increasing the horizon from $5{,}000$ to $50{,}000$ further reduces the mean for both methods in all six instances, also sort of predictable by the rate $O(\nicefrac{\log{T}}{T})$. The variance exhibits a different pattern. On Bernoulli rewards with $\eta_\mu=0.01$, the mean unfairness of ATT decreases from $0.3041$ to $0.1368$, whereas its variance increases from $0.00904$ to $0.0282$. Hypothetically, this may be due to that the horizon $T$ has not reached the sufficient horizon for the upper bound on variance. 

\paragraph{The role of quota protection.}
Quota protection helps balance the allocation among Pareto-optimal
arms, but can also incur additional pulls of dominated arms.
The relative importance of these two effects determines its
impact on mean unfairness and variance.
For Bernoulli rewards at $T=50{,}000$, EG-ATT reduces the
variance by $77\%$ at $\eta_\mu=0.01$, with similar mean unfairness,
and improves both metrics at $\eta_\mu=0.1$. Likewise, for gaussian rewards, EG-ATT also improves both metrics at
$\eta_\mu=0.02$. Conversely, in some cases, for example at $\eta_\mu=0.1$, it achieves nearly
zero variance at the cost of larger mean unfairness.

\paragraph{Effect of the BO objective.}
The tuning objective also shapes the trade-off of mean and variance we observe. The distinction is particularly visible at $T=5{,}000$. On Bernoulli rewards with $\eta_\mu=0.01$, log-tuned EG-ATT attains variance $7.50\times10^{-10}$ but mean unfairness $0.4053$, whereas mean$+$SD tuning yields mean $0.3020$ and variance $0.0100$.  We observe a similar outcome for Gaussian rewards at $\eta_\mu=0.02$: log-tuned EG-ATT has variance $6.07\times10^{-15}$ and mean unfairness $0.2196$, compared with $0.0020$ and $0.0390$ under mean$+$SD tuning. 

\paragraph{Comparison with existing methods.} Among all the methods, our newly proposed methods almost always outperform, with only one exception where the new methods and PTS are more or less the same. In many cases, it achieves an order-of-magnitude improvement. Across $120$ pairwise comparisons with the five baselines, they reduce mean unfairness by more than
$50\%$ in $60$ cases ($50.0\%$), and by more than $90\%$
in $37$ cases ($30.8\%$); e.g., a $96.9\%$ reduction
relative to Annealing-Pareto on Bernoulli rewards with
$\eta_\mu=0.01$ at $T=50{,}000$. Particularly, PTS is competitive when the separation margin and the time horizon are large. For Bernoulli rewards with $\eta_\mu=0.1$ and $T=5{,}000$, it achieves the smallest mean unfairness, $0.0222$, among all; the proposed method achieves $0.0227$. At $T=50{,}000$, its mean is $0.00325$ on Bernoulli rewards and $5.18\times10^{-4}$ on Gaussian rewards. The hypothesis is that by sampling from a posterior Pareto set, PTS can spread pulls among Pareto-optimal arms once their membership becomes clearer which is more feasible with a larger separation margin. Meanwhile, Pareto-KG can achieve a similar mean, such as $0.1436$ on Bernoulli rewards with $\eta_\mu=0.01$ at $T=50{,}000$, but its variance $0.0457$ is much larger. It samples uniformly from a Pareto set adjusted by
knowledge-gradient bonuses, which can yield competitive mean fairness.
However, reward fluctuations affect which arms enter this set and
for how long, thereby yielding larger fluctuations in fairness across runs.

On the contrary, Pareto-UCB1 often has a small variance, attaining the smallest variance among the methods in the table on the two more difficult Bernoulli instances at $T=50{,}000$. Its mean unfairness, however, remains larger than that of the proposed methods. By design, it samples uniformly from an estimated Pareto set using their upper confidence estimates. This tends to stabilize the allocation as the estimates can converge nicely, but the set itself may also contain suboptimal arms. By doing so, it also considers the fairness, though only in part, which explains why it is not the worst. Differently, Width-guided UCB provides a more extreme example. It aims to aggressively select one arm that wins on an objective, en route to minimizing Pareto regret by design. The price of such aggressiveness is that it mostly fails to  balance the pulls among all Pareto-optimal arms. Such price only gets higher as the horizon increases. On Bernoulli rewards with $\eta_\mu=0.1$, its mean unfairness increases from $0.2208$ at $T=5{,}000$ to $1.6001$ at $T=50{,}000$. Annealing-Pareto is among the worst-performing. It incurs much larger unfairness on the Bernoulli instances, while its performance improves substantially on the Gaussian instances as the separation increases. Interestingly, its performance mostly downgrades as the time horizon increases, consistent with the Width-guided UCB but as opposed to the majority of the methods. This certainly makes it less favorable, however, with a clear reason. The algorithm aims to minimize the Shannon entropy based on their measure; as introduced earlier, this is flawed by definition.

\section{Conclusion and AI Use Statement}\label{final}

Stochastic multi-objective bandits have shown promise in online advertising where a single objective can be far from sufficient for decision makers. With the progress of the community, we have a decent understanding of how to identify and select Pareto optimal arms. However, it seems that we know much less about how to select among these optimal arms in a balanced manner. One barrier is that there seems to be a lack of consensus on a proper measure for it, and another barrier is that the set of Pareto optimal arms is typically unknown. 
Surmounting these barriers requires more careful control
of the sampling trajectory while learning the Pareto set.
We take steps herein toward this goal.
 We first introduce a measure for unfairness based on the KL divergence from the ideal measure, the uniform distribution over the Pareto optimal arms, to the actual measure, the empirical measure of the sampling trajectory induced by a policy. It resembles several existing measures or corrects their limitations while remaining well defined in any circumstance. An exact decomposition of this measure further provides pathways for method design. We propose an adaptive top-two sampling method, which defines a gain criterion for each arm at each time step based on a soft target for the true fairness gain, and randomly selects between the top two arms to promote exploration. The upper bound of the method exhibits a satisfactory rate for the expected unfairness, $O(\nicefrac{\log{T}}{T})$, whereas the variance seems to pose computational challenges. Hence, we further propose a variant aiming to reduce variance, which adds more exploration at earlier stages to avoid misidentified Pareto membership and thus fluctuations in fairness. Numerical experiments also help interpret the theory, show the advantages of the proposed methods, and offer practical insights into existing methods.

Generative AI tools, particularly GPT-6, have assisted the authors with coding, drafting the proofs of the mathematical claims, manuscript polishing, and reviewing the content. The tools did not participate in the ideation, conceptualization, or formalization of the paper. They also helped with the rigorous steps in the method design after the authors sketched the rough picture of the method. The authors have independently revised the content and take responsibility for the final content of this work, including text, claims, or artifacts produced with the aid of generative AI.

\clearpage 
\newpage 

\bibliographystyle{plainnat}
\bibliography{iclr2027_conference}

\clearpage 
\newpage 

\appendix

The appendix proceeds as follows.
Section~\ref{app:ungated-proofs} presents the proofs of our theoretical results. Since fairness essentially depends on the sampling process, it is in fact more convenient to start with the sampling process, such as arm pulls and the imbalance potential.
Hence, we first present the proofs for bounds on arm pulls and imbalance, followed by those for the bounds
on expected unfairness and variance. This is in the reverse order of the results presented in the main body, and we believe this exposition is natural, as the two orders serve different purposes. We then derive the proof for Pareto regret.
We then prove the information-theoretic lower bounds, including the asymptotic
and finite-time bounds under positive Pareto separation and the lower bound
at zero separation. Finally, we establish the guarantees for EG-ATT designed for variance reduction. At the same time, to enhance the reproducibility of the results, 
Section~\ref{app:fairness-instance-details} provides the reward instances,
baseline implementations, and tuning details for the numerical experiments.

\section{Proofs}
\label{app:ungated-proofs}

\subsection{Proof for Arm Pulls and Imbalance}

\begin{proof}[Proof of Theorem~\ref{thm:ungated-moments}]
We first control inaccurate rounds, then use this control
to bound dominated-arm pulls and the imbalance potential.

\emph{Step 1: exploration and estimation errors.}
For this proof, set
\[
\alpha=\frac{(64+32s)\sigma^2}{\eta_\mu^2},
\qquad
t_*=\left\lceil\frac{\sqrt{T_0}}{2}\right\rceil.
\]
Let $\mathcal E_t$ be the event that
\[
\min_iN_i(t)\ge\alpha\log(t+1),
\qquad
\max_{i,r}
|\widetilde\mu_{t,i}^{(r)}-\mu_i^{(r)}|
\le\eta_\mu/4.
\]
We will show that
\begin{equation}\label{eq:ungated-error-sum}
2\sum_{t=2K}^{\infty}
(t+1)\Prob_\mu(\mathcal E_t^c)
\le T_0.
\end{equation}

When $\widehat\Delta_{t,i}>0$, changing only arm $i$'s
projected mean to the upper corner of $\mathcal M$
makes that arm Pareto optimal. Hence
\[
Z_{t,i}\le
\frac{N_i(t)B_{\mathcal M}^2}{2\sigma^2}.
\]
Using $\widehat\Delta_{t,i}\le B_{\mathcal M}$ and
$\sqrt{x^2+x}\le x+1/2$ gives
\begin{equation}\label{eq:ungated-exponent}
\frac{R_{t,i}\sqrt{1+Z_{t,i}}}{c_\tau}
\le
\frac{B_{\mathcal M}^2N_i(t)}
{\sqrt{2}\sigma^2c_\tau}
+\frac{1}{\sqrt{2}c_\tau}.
\end{equation}
When $R_{t,i}=0$, the exponent is zero by convention,
and the same bound holds.

Our choice of $\alpha$ satisfies
\[
\frac{\alpha B_{\mathcal M}^2}
{\sqrt{2}\sigma^2c_\tau}
=
(1-s)(1+s/2)
\le1-s/2,
\qquad
\frac{\alpha\eta_\mu^2}{32\sigma^2}=2+s.
\]
Also, $\eta_\mu\le B_{\mathcal M}$ and
\eqref{eq:large-temp} imply
$c_\tau>32\sqrt{2}$, so
$\exp\{1/(\sqrt{2}c_\tau)\}<2$.

Fix $t\ge t_*$ and suppose that
$N_a(t)<\alpha\log(t+1)$ for some arm $a$.
For every earlier round $\ell$, monotonicity of the counts
and \eqref{eq:ungated-exponent} imply
\[
G_{\ell,a}
\ge
\frac{(t+1)^{-1+s/2}}
{2(1+\alpha\log(t+1))}.
\]
Since the leader maximizes the gain and
$G_{\ell,b_\ell}\le1/N_{b_\ell}(\ell)$,
\[
N_{b_\ell}(\ell)
\le
2(1+\alpha\log(t+1))(t+1)^{1-s/2}.
\]
Thus, over
$\ell=\lceil t/2\rceil,\ldots,t-1$, the total number of
pulls made as the leader is at most
\[
K\left[
2(1+\alpha\log(t+1))(t+1)^{1-s/2}+2
\right]
\le \frac{t}{16}.
\]
For the last inequality, use
$\alpha\le96\sigma^2/\eta_\mu^2$,
$\log(t+1)\le4(t+1)^{s/4}/s$, and
\[
t\ge t_*
\ge
\left[
128K\left(1+\frac{384\sigma^2}{s\eta_\mu^2}\right)
\right]^{4/s},
\]
which follows from \eqref{eq:ungated-horizon}.

At every round, the leader is selected with probability
at least $0.45$, and therefore at least $2/5$.
Using independent uniform random variables to implement
these selections, the number of leader pulls over these
rounds is bounded below by a binomial random variable
with parameters $\lfloor t/2\rfloor$ and $2/5$.
A Chernoff bound consequently gives
\[
\Prob_\mu\left\{
\min_iN_i(t)<\alpha\log(t+1)
\right\}
\le e^{-t/64},
\qquad t\ge t_*.
\]

Projection onto $\mathcal M$ cannot increase the error
relative to the true mean. Applying the sub-Gaussian
bound to the reward tapes and summing over all sample
sizes $n\ge\lceil\alpha\log(t+1)\rceil$ yields
\begin{align*}
\Prob_\mu(\mathcal E_t^c)
&\le
e^{-t/64}
+
\frac{2Kd}{1-e^{-\eta_\mu^2/(32\sigma^2)}}
(t+1)^{-2-s}\\
&\le
e^{-t/64}
+
2Kd\left(1+\frac{32\sigma^2}{\eta_\mu^2}\right)
(t+1)^{-2-s},
\qquad t\ge t_*.
\end{align*}
Here we used
$(1-e^{-x})^{-1}\le1+x^{-1}$ for $x>0$.
For earlier rounds, use the trivial probability bound
of one. Therefore,
\begin{align*}
2\sum_{t=2K}^{\infty}
(t+1)\Prob_\mu(\mathcal E_t^c)
&\le
t_*(t_*+1)+8450
+\frac{4Kd}{s}
\left(1+\frac{32\sigma^2}{\eta_\mu^2}\right)\\
&\le T_0.
\end{align*}
Indeed, the first term is at most $T_0/2$, and each of
the remaining terms is at most $T_0/4$, by the explicit
definition of $T_0$. The numerical bound $8450$ follows
from
\[
2\sum_{t=0}^{\infty}(t+1)e^{-t/64}
=
\frac{2}{(1-e^{-1/64})^2}
\le2(1+64)^2.
\]
This proves \eqref{eq:ungated-error-sum}.

Let
\[
M_T=\sum_{t=2K}^{T-1}\ind\{\mathcal E_t^c\}
\]
count inaccurate rounds. Expanding the square and using
$\Prob_\mu(\mathcal E_u^c\cap\mathcal E_t^c)
\le\Prob_\mu(\mathcal E_t^c)$ for $u\le t$, we obtain
\begin{equation}\label{eq:ungated-M2}
\E_\mu M_T^2
\le
2\sum_{t=2K}^{T-1}
(t+1)\Prob_\mu(\mathcal E_t^c)
\le T_0.
\end{equation}

\emph{Step 2: dominated-arm pulls.}
On $\mathcal E_t$, the empirical Pareto set equals
$\Astar$, and every true Pareto arm has gain
$\log(1+1/N_i(t))$.

For each dominated arm $i$, choose a Pareto optimal arm
$p_i$ such that
\[
\mu_{p_i}^{(r)}-\mu_i^{(r)}\ge\eta_\mu
\qquad\text{for every }r.
\]
Such an arm exists by the definition of
$\eta_\mu^{\rm D}$: if a dominating arm is not Pareto
optimal, replacing it by an arm that dominates it
cannot decrease any coordinatewise advantage.

On $\mathcal E_t$, each empirical advantage of $p_i$
over $i$ is at least $\eta_\mu/2$. Consequently,
\[
R_{t,i}
\ge\frac{\eta_\mu\sqrt{N_i(t)}}{2\sigma},
\qquad
Z_{t,i}
\ge
\frac{\eta_\mu^2}{8\sigma^2}
\frac{N_i(t)N_{p_i}(t)}
{N_i(t)+N_{p_i}(t)}.
\]
For the second inequality, any alternative that makes
arm $i$ Pareto optimal must break its domination by
$p_i$. If the two mean vectors change by amounts $u$
and $v$ in the infinity norm, this requires
$u+v\ge\eta_\mu/2$. Minimizing the corresponding weighted
sum of squares gives the stated bound. It follows that
\[
\frac{R_{t,i}\sqrt{1+Z_{t,i}}}{c_\tau}
\ge
\frac{\eta_\mu^2}{4\sqrt{2}c_\tau\sigma^2}
N_i(t)
\sqrt{\frac{N_{p_i}(t)}
{N_i(t)+N_{p_i}(t)}}.
\]

If a dominated arm is selected on $\mathcal E_t$, its
gain is at least $1/(t+1)$. When $m\ge2$, two true
Pareto arms have gains at least this large. When $m=1$,
the singleton rule selects the leader, whose gain is
at least that of the true Pareto arm. Since
$\log(1+1/N_i(t))\le1$, the exponent in its gain is
therefore at most $\log(t+1)$.

Using $N_{p_i}(t)\ge\alpha\log(t+1)$ and considering
$N_i(t)\le N_{p_i}(t)$ and
$N_i(t)>N_{p_i}(t)$ separately gives
\begin{align*}
N_i(t)
&\le
\max\left\{
\frac{8c_\tau\sigma^2}{\eta_\mu^2},
\frac{64c_\tau^2\sigma^4}{\alpha\eta_\mu^4}
\right\}\log(t+1)\\
&\le
\frac{c_\tau^2\sigma^2}{\eta_\mu^2}\log(t+1).
\end{align*}
The last inequality uses
$\alpha\ge64\sigma^2/\eta_\mu^2$ and $c_\tau>8$.

Accounting for initialization, the last pull on an
accurate round, and pulls on inaccurate rounds yields
the pathwise bound
\begin{equation}\label{eq:ungated-Bbound}
B_T
\le
(K-m)\left(
3+\frac{c_\tau^2\sigma^2}{\eta_\mu^2}\log(T+1)
\right)+M_T.
\end{equation}
Taking the $L^2$ norm and applying
\eqref{eq:ungated-M2} proves
\eqref{eq:ungated-B-moments}.

\emph{Step 3: imbalance among Pareto optimal arms.}
When $m=1$, $\Phi_T=0$. Suppose $m\ge2$ and write
\[
H_t=\sum_{i\in\Astar}N_i(t),
\qquad
x_i(t)=N_i(t)-H_t/m.
\]
A pull of Pareto arm $i$ changes the potential by
\[
\Phi_{t+1}-\Phi_t
=
2x_i(t)+1-\frac1m,
\]
whereas a dominated-arm pull leaves it unchanged.

On $\mathcal E_t$, if exactly one candidate is Pareto
optimal, it is a least-sampled true Pareto arm and has
a nonpositive centered count. If both candidates are
Pareto optimal, they are the two least-sampled true
Pareto arms. Their centered counts sum to at most zero,
and the smaller-count candidate is selected with
probability at least one half. Thus,
\[
\E_\mu[
\Phi_{t+1}-\Phi_t\mid\mathscr F_t
]\le1
\qquad\text{on }\mathcal E_t.
\]
On any round,
$|\Phi_{t+1}-\Phi_t|\le2t+1$.
Since $\Phi_{2K}=0$, \eqref{eq:ungated-error-sum} gives
\[
\E_\mu\Phi_T
\le
T+\sum_{t=2K}^{T-1}
(2t+1)\Prob_\mu(\mathcal E_t^c)
\le T+T_0.
\]
Finally,
\[
0\le\Phi_T
=
\sum_{i\in\Astar}N_i(T)^2-\frac{H_T^2}{m}
\le H_T^2\le T^2.
\]
Therefore,
\[
\E_\mu\Phi_T^2
\le T^2\E_\mu\Phi_T
\le T^2(T+T_0),
\]
which completes the proof.
\end{proof}

\subsection{Proof for the Expected Fairness and Variance}

\begin{proof}[Proof of Theorem~\ref{thm:ungated}]
Fix $T\ge T_0$ and let $n=T-2K$.
The definition of $T_0$ ensures $T\ge12K$, so
$n\ge5T/6$. It also gives
\[
(K-m)\left(
3+\frac{c_\tau^2\sigma^2}{\eta_\mu^2}\log(T+1)
\right)
\le\frac{T}{16}\le\frac n8,
\]
using $\log(T+1)\le\sqrt T$.
Hence \eqref{eq:ungated-Bbound} implies
\[
B_T\le n/8+M_T.
\]

We first verify that every Pareto optimal arm receives
a positive fraction of the observations, except on an
event whose contribution to the fairness moments is
small.

Suppose $m\ge2$. Let $J_t$ indicate that $\mathcal E_t$
holds and at least one candidate is a true Pareto arm.
Let $I_t$ indicate that the selected arm is a currently
least-sampled true Pareto arm.
On a round counted by $J_t$, the probability of such
a selection is at least $0.30$. Indeed, when both
candidates are Pareto optimal, the leader is least
sampled and is selected with probability at least
$0.70$; when exactly one is Pareto optimal, it is least
sampled and is selected with probability $0.30$ or
$0.70$. Therefore,
\[
\E_\mu[I_t\mid\mathscr F_t]\ge0.30J_t.
\]
Moreover,
\[
\sum_{t=2K}^{T-1}J_t
\ge n-B_T-M_T
\ge\frac{7n}{8}-2M_T.
\]
On $\{M_T\le n/16\}$, this sum is at least $3n/4$.
By \eqref{eq:ungated-M2}, Markov's inequality, and
Azuma's inequality,
\[
\Prob_\mu\left\{
\sum_{t=2K}^{T-1}I_t<\frac{3n}{20}
\right\}
\le
\frac{256T_0}{n^2}
+\exp\left\{-\frac{9n}{3200}\right\}.
\]

Starting from two pulls per arm, at most $m$ selections
of a currently least-sampled Pareto arm are needed to
increase the minimum Pareto-arm count by one.
Other selections cannot decrease any count. Thus,
\[
\min_{i\in\Astar}N_i(T)
\ge
2+\left\lfloor
\frac1m\sum_{t=2K}^{T-1}I_t
\right\rfloor.
\]
If the sum is at least $3n/20$, then
\[
\min_{i\in\Astar}N_i(T)
\ge\frac{3n}{20m}
\ge\frac{T}{8m}.
\]

When $m=1$, the same count bound follows directly on
$\{M_T\le n/16\}$ from
\[
N_p(T)=T-B_T
\ge T-\frac n8-\frac n{16}
\ge\frac{13T}{16},
\]
where $p$ is the unique Pareto optimal arm.

Consequently, for
\[
\mathcal C_T=
\left\{
\min_{i\in\Astar}N_i(T)\ge\frac{T}{8m}
\right\},
\]
we have
\begin{equation}\label{eq:ungated-count-event}
\Prob_\mu(\mathcal C_T^c)
\le
\frac{256T_0}{n^2}
+\exp\left\{-\frac{9n}{3200}\right\}
\le\frac{512T_0}{T^2}.
\end{equation}
For the final inequality, use $n\ge5T/6$ and
\[
e^{-9n/3200}
\le e^{-3T/1280}
\le\frac{2(1280/3)^2}{T^2}
\le\frac{T_0}{T^2},
\]
where the last inequality follows from the explicit
size of $T_0$.

On $\mathcal C_T$, we have $H_T=T-B_T\ge T/8$.
The exact decomposition and its chi-square bound give
\begin{align*}
\F_T
&\le
\frac{B_T}{H_T}
+\frac1{H_T}
\sum_{i\in\Astar}
\frac{(N_i(T)-H_T/m)^2}{N_i(T)}\\
&\le
\frac{8B_T}{T}+\frac{64m\Phi_T}{T^2}.
\end{align*}
On its complement, initialization gives
$0\le\F_T\le\log(T/m)$.
Splitting the expectation over these events yields
\[
\E_\mu\F_T
\le
\frac{8\E_\mu B_T}{T}
+\frac{64m\E_\mu\Phi_T}{T^2}
+\frac{512T_0}{T^2}\log\frac{T}{m}.
\]
Similarly, using $\Var_\mu(\F_T)\le\E_\mu\F_T^2$,
\[
\Var_\mu(\F_T)
\le
\frac{128\E_\mu B_T^2}{T^2}
+\frac{8192m^2\E_\mu\Phi_T^2}{T^4}
+\frac{512T_0}{T^2}\log^2\frac{T}{m}.
\]
Substitute Theorem~\ref{thm:ungated-moments} and use
$T+T_0\le2T$ to obtain
\eqref{eq:ungated-mean} and
\eqref{eq:ungated-variance}.

Finally, convexity of KL divergence in its second
argument gives
\[
\bar{\mathcal F}_T
=
\KL(\pi^\star\|\E_\mu\widehat\pi_T)
\le
\E_\mu\KL(\pi^\star\|\widehat\pi_T)
=
\E_\mu\F_T.
\]
This completes the proof.
\end{proof}

\subsection{Proof for the Pareto Regret}

\begin{proof}
We retain the individual arm gaps in the proof of
Theorem~\ref{thm:ungated-moments}.
Let $\mathcal E_t$ and $M_T$ be the accurate-round event
and the number of inaccurate rounds defined there.
In particular, $\E_\mu M_T^2\le T_0$.

For a dominated arm $i$, choose a Pareto optimal arm $p_i$
attaining its domination margin:
\[
\min_{r\in[d]}
\bigl(\mu_{p_i}^{(r)}-\mu_i^{(r)}\bigr)
=
\Delta_i^{\rm P}.
\]
Such a maximizing arm can be chosen in $\Astar$, since
replacing a dominated maximizer by an arm that dominates
it cannot decrease any coordinatewise advantage.
On $\mathcal E_t$, every empirical advantage of $p_i$
over $i$ is at least
\[
\Delta_i^{\rm P}-\frac{\eta_\mu}{2}
\ge\frac{\Delta_i^{\rm P}}2.
\]
Therefore, the gain comparison in Step~2 of the arm-count
proof applies with $\Delta_i^{\rm P}$ in place of
$\eta_\mu$. Whenever arm $i$ is selected on $\mathcal E_t$,
it gives
\begin{align*}
N_i(t)
&\le
\sigma^2
\max\left\{
\frac{8c_\tau}{(\Delta_i^{\rm P})^2},
\frac{c_\tau^2\eta_\mu^2}{(\Delta_i^{\rm P})^4}
\right\}\log(t+1)\\
&\le
\frac{c_\tau^2\sigma^2}{(\Delta_i^{\rm P})^2}
\log(t+1).
\end{align*}
The last inequality uses $\Delta_i^{\rm P}\ge\eta_\mu$
and $c_\tau>8$, which follows from the temperature
condition and $\eta_\mu\le B_{\mathcal M}$.

Accounting for initialization and the last pull on an
accurate round, and charging each inaccurate round at
most $\Delta_{\max}^{\rm P}$ regret, yields
\[
R_T^{\rm P}
\le
3\sum_{i\notin\Astar}\Delta_i^{\rm P}
+
c_\tau^2\sigma^2\log(T+1)
\sum_{i\notin\Astar}\frac1{\Delta_i^{\rm P}}
+
\Delta_{\max}^{\rm P}M_T.
\]
Taking expectations and using
$\E_\mu M_T\le\sqrt{\E_\mu M_T^2}\le\sqrt{T_0}$
proves \eqref{eq:ungated-pareto-regret}.

Finally,
$R_T^{\rm P}\le\Delta_{\max}^{\rm P}B_T$
follows directly from the definition of regret.
The exact unfairness decomposition gives
$\F_T\ge B_T/T$, proving
\eqref{eq:regret-fairness}.
\end{proof}

\subsection{Proofs for the Lower Bounds}
\label{app:lower-proofs}

\begin{proof}[Proof of Theorem~\ref{thm:lower-asymptotic}]
Fix a dominated arm $i$ and an admissible alternative $Q_i$
such that the resulting instance $\nu^{(i)}$ satisfies
\[
\eta_{\nu^{(i)}}>0,
\qquad
i\in\Astar(\nu^{(i)}),
\qquad
0<\KL(P_i\|Q_i)<\infty.
\]
For notation convenience, we write $n_i=\E_\mu N_i(T)$ and
$E_i=\{N_i(T)\le\sqrt T\}$.

We first establish an event bound under the alternative instance. 
Let $m_i'=|\Astar(\nu^{(i)})|\le K$. By data processing for KL divergence, on the event $E_i$,
\begin{align*}
\F_T
&\ge
\kl\left(\frac1{m_i'},\frac{N_i(T)}T\right)\\
&\ge
\frac1{m_i'}\log\frac{T}{N_i(T)}-\log2\\
&\ge
\frac{\log T}{2K}-\log2 \ge
\frac{\log T}{4K},
\end{align*}
for all sufficiently large horizons. Here, the second inequality follows from the fact that binary entropy is at most $\log2$. Markov's inequality therefore gives
\[
\Prob_{\nu^{(i)}}(E_i)
\le
\frac{4K}{\log T}\E_{\nu^{(i)}}\F_T.
\]
For every fixed $\zeta\in(0,1)$, uniform fairness
implies
\[
\E_{\nu^{(i)}}\F_T=o(T^{-1+\zeta}).
\]
Consequently,
\[
\Prob_{\nu^{(i)}}(E_i)\le T^{-1+\zeta}
\]
for all sufficiently large $T$.

We next apply the change-of-measure argument explicitly.

Since arm $i$ is dominated under $\mu$ and the policy is uniformly fair, 
\[
n_i
\le T\E_\mu\F_T
=o(T^{1/4}).
\]
By Markov's inequality, 
\[
\Prob_\mu(E_i)
\ge 1-\frac{n_i}{\sqrt T}
\longrightarrow 1.
\]
Moreover, the adaptive KL chain rule gives 
\[
\KL(\Prob_\mu^T\|\Prob_{\nu^{(i)}}^T)
=
n_i\KL(P_i\|Q_i)<\infty,
\]
which implies 
$\Prob_\mu^T\ll\Prob_{\nu^{(i)}}^T$.
Therefore,
$\Prob_{\nu^{(i)}}(E_i)>0$
for all sufficiently large $T$.

When the alternative event has positive probability, i.e., $\Prob_{\nu^{(i)}}(E_i)>0$, the adaptive KL chain rule and data processing give
\begin{align*}
n_i\KL(P_i\|Q_i)
&\ge
\kl\bigl(
\Prob_\mu(E_i),
\Prob_{\nu^{(i)}}(E_i)
\bigr)\\
&\ge
\Prob_\mu(E_i)
\log\frac1{\Prob_{\nu^{(i)}}(E_i)}
-\log2.
\end{align*}
The second inequality again uses the binary entropy bound. 
Hence, we obtain, for all sufficiently large horizons, 
\[
n_i\KL(P_i\|Q_i)
\ge
\left(1-\frac{n_i}{\sqrt T}\right)
(1-\zeta)\log T-\log2.
\]
Thus,
\[
\E_\mu N_i(T)
\ge
\frac{
(1-\zeta)\log T-\log2
}{
\KL(P_i\|Q_i)
+\dfrac{(1-\zeta)\log T}{\sqrt T}
}.
\]

Dividing by $\log T$ and taking lower limits yields
\[
\liminf_{T\to\infty}
\frac{\E_\mu N_i(T)}{\log T}
\ge
\frac{1-\zeta}{\KL(P_i\|Q_i)}.
\]
Letting $\zeta\downarrow0$ and taking the supremum over
all admissible finite-KL alternatives gives
\[
\liminf_{T\to\infty}
\frac{\E_\mu N_i(T)}{\log T}
\ge
\frac1{\kappa_i^{\rm KL}(\mu)}.
\]

Every admissible alternative has strictly positive
KL divergence. Indeed, if 
$\KL(P_i\|Q_i)=0$,
then $P_i=Q_i$ by Pinsker's inequality.
Since the distributions of all other arms remain
unchanged, the resulting instance coincides with 
$\mu$, contradicting 
$i\in\Astar(\nu^{(i)})$ and $i\notin\Astar(\mu)$.

Finally, the unfairness decomposition implies
\[
T\E_\mu\F_T
\ge
\sum_{i\notin\Astar(\mu)}\E_\mu N_i(T).
\]
Since the sum is finite and its terms are nonnegative,
\begin{align*}
\liminf_{T\to\infty}
\frac{T\E_\mu\F_T}{\log T}
&\ge
\sum_{i\notin\Astar(\mu)}
\liminf_{T\to\infty}
\frac{\E_\mu N_i(T)}{\log T}\\
&\ge
\sum_{i\notin\Astar(\mu)}
\frac1{\kappa_i^{\rm KL}(\mu)}.
\end{align*}
This proves \eqref{eq:lower-asymptotic}.
\end{proof}

\begin{proof}[Proof of Theorem~\ref{thm:lower-finite}]
Again, for each arm in the chosen subset, we highlight that we use an alternative that changes only its reward distribution, makes it Pareto optimal, and satisfies
\[
\eta_{\nu^{(i)}}>0,\qquad
i\in\Astar(\nu^{(i)}),\qquad
0<\KL(P_i\|Q_i)<\infty.
\]
Let
\[
r_T=
\max\left\{
\E_\mu\F_T,\,
\max_{i\in\mathcal I}\E_{\nu^{(i)}}\F_T
\right\}.
\]
The result is immediate if $r_T=\infty$, so suppose
$r_T<\infty$.

For each $i\in\mathcal I$, consider the event
\[
E_i=\{N_i(T)\le T/(2K)\}.
\]
Since arm $i$ is dominated under $\mu$ and
$\F_T\ge N_i(T)/T$, Markov's inequality gives
\[
\Prob_\mu(E_i^c)
\le
\frac{2K}{T}\E_\mu N_i(T)
\le 2Kr_T.
\]

Under $\nu^{(i)}$, arm $i$ becomes Pareto optimal.
Writing $m_i'=|\Astar(\nu^{(i)})|\le K$,
data processing and Pinsker's inequality give
\[
\F_T
\ge
\kl\left(\frac1{m_i'},\frac{N_i(T)}T\right)
\ge
2\left(\frac1{m_i'}-\frac{N_i(T)}T\right)^2.
\]
On $E_i$, we have
\[
\frac1{m_i'}-\frac{N_i(T)}T
\ge\frac1{2K},
\]
and therefore $\F_T\ge1/(2K^2)$.
Consequently,
\[
\Prob_{\nu^{(i)}}(E_i)\le2K^2r_T.
\]

The instances $\mu$ and $\nu^{(i)}$ differ only
in the reward distribution of arm $i$.
The adaptive KL chain rule and the
Bretagnolle--Huber inequality (Lemma 7 in \citet{duchi2024information}) imply
\begin{align*}
&\Prob_\mu(E_i^c)+\Prob_{\nu^{(i)}}(E_i)\\
&\qquad\ge
\frac12\exp\left(
-\KL(P_i\|Q_i)\E_\mu N_i(T)
\right).
\end{align*}

Combining the preceding bounds yields
\[
2K(K+1)r_T
\ge
\frac12\exp\left(
-\KL(P_i\|Q_i)\E_\mu N_i(T)
\right).
\]
Since $\KL(P_i\|Q_i)<\infty$
and $\E_\mu N_i(T)\le T<\infty$, the right-hand side is strictly positive,
so $r_T>0$.

Taking logarithms gives
\[
\E_\mu N_i(T)
\ge
\frac1{\KL(P_i\|Q_i)}
\log\frac1{4K(K+1)r_T}.
\]

Summing over the arms and using the pathwise inequality
\[
T\F_T
\ge B_T
\ge\sum_{i\in\mathcal I}N_i(T),
\]
we obtain, after taking expectations, 
\[
Tr_T
\ge
\sum_{i\in\mathcal I}\E_\mu N_i(T)
\ge
S\log\frac1{4K(K+1)r_T}.
\]
Let
\[
x=\frac{Tr_T}{S},
\qquad
z=\frac{T}{4K(K+1)S}.
\]
The preceding inequality gives $xe^x\ge z$.
By assumption, $z\ge e$. Moreover,
\[
(\log z-\log\log z)
e^{\log z-\log\log z}
=
z\left(1-\frac{\log\log z}{\log z}\right)
\le z.
\]
Since $xe^x$ is increasing for $x\ge0$,
we conclude that
\[
x\ge\log z-\log\log z.
\]
Substituting the definitions of $x$ and $z$
proves \eqref{eq:lower-fair-finite}.
\end{proof}

\begin{proof}[Proof of Theorem~\ref{thm:boundary-superlog}]
Let $x_T=\E_{\nu_0}N_1(T)$ and
$E_T=\{N_1(T)\ge T/4\}$.
Markov's inequality gives
\[
\Prob_{\nu_0}(E_T)\le\frac{4x_T}{T}.
\]

Under $\nu_\varepsilon$, both arms are Pareto optimal.
On $E_T^c$, their allocation is sufficiently unbalanced
that
\[
\F_T
=
\frac12\log
\frac{T^2}{4N_1(T)(T-N_1(T))}
\ge
\kappa:=\frac12\log\frac43>0.
\]
Since $\nu_\varepsilon$ has positive Pareto separation,
uniform fairness implies that, for every fixed
$\varepsilon\in(0,1)$ and $\zeta\in(0,1)$,
\[
\Prob_{\nu_\varepsilon}(E_T^c)
\le
\frac{\E_{\nu_\varepsilon}\F_T}{\kappa}
=
o(T^{-1+\zeta}).
\]

Let $\Prob_\nu^T$ denote the law of the actions and observed
rewards through time $T$ under instance $\nu$.
The policy uses the same conditional action probabilities
under both instances, and only arm 1's second reward
coordinate changes. The chain rule for KL divergence
therefore gives
\[
\KL(\Prob_{\nu_0}^T\|\Prob_{\nu_\varepsilon}^T)
=
\frac{\varepsilon^2}{2\sigma^2}
\E_{\nu_0}N_1(T)
=
\frac{\varepsilon^2x_T}{2\sigma^2}.
\]
By the Bretagnolle--Huber inequality
(Lemma 7 in \citet{duchi2024information}),
\[
\Prob_{\nu_0}(E_T)
+\Prob_{\nu_\varepsilon}(E_T^c)
\ge
\frac12
\exp\left\{-\frac{\varepsilon^2x_T}{2\sigma^2}\right\}.
\]
Combining the preceding bounds yields
\[
\frac{4x_T}{T}
+o(T^{-1+\zeta})
\ge
\frac12
\exp\left\{-\frac{\varepsilon^2x_T}{2\sigma^2}\right\}
\]
for every fixed $\varepsilon\in(0,1)$ and
$\zeta\in(0,1)$.

Suppose that $x_T/\log T$ did not tend to infinity.
Then there would exist a finite $M>0$ and a sequence
$T_k\to\infty$ such that $x_{T_k}\le M\log T_k$.
Choose a fixed $\varepsilon\in(0,1)$ sufficiently
small that
\[
a:=\frac{\varepsilon^2M}{2\sigma^2}<1,
\]
and then choose $\zeta\in(0,1-a)$.
Along this sequence, the preceding inequality implies
\[
\frac12
\le
\frac{4M\log T_k}{T_k^{1-a}}
+
o\left(T_k^{-(1-a-\zeta)}\right).
\]
Both terms on the right-hand side tend to zero,
a contradiction. Therefore,
\[
\frac{\E_{\nu_0}N_1(T)}{\log T}
\longrightarrow\infty.
\]

Finally, under $\nu_0$, only arm 2 is Pareto optimal,
so
\[
\F_T
=
-\log\left(1-\frac{N_1(T)}{T}\right)
\ge\frac{N_1(T)}{T}.
\]
Taking expectations and using the preceding limit gives
\[
\frac{T\E_{\nu_0}\F_T}{\log T}
\longrightarrow\infty,
\]
completing the proof.
\end{proof}

\subsection{Proof for Results in Section \ref{sec:safeguard}}

\begin{proof}
We first show that the evidence gate is released within
$O_\mu(\log T)$ rounds with high probability. We then
bound the imbalance among Pareto optimal arms after
release and use these bounds to control fairness.

\paragraph{Release of the quota.}
Let $\mathcal E_T$ be the event on which, for every
arm, objective, and sample count $2\le n\le T$,
the empirical mean differs from its true mean by
at most
\[
r(n)=\sqrt{\frac{2\sigma^2\Lambda_T}{n}}.
\]
Since projection onto $\mathcal M$ does not increase
the estimation error, sub-Gaussian concentration
and a union bound over the arm reward tapes give
\[
\Prob_\mu(\mathcal E_T^c)
\le 2KdT e^{-\Lambda_T}
=\frac{1}{2T^6}.
\]
The event remains valid under adaptive sampling.

For the analysis, let
\[
n_\mu=
\max\left\{
2,\left\lceil
\frac{512\sigma^2\Lambda_T}{\eta_\mu^2}
\right\rceil
\right\}.
\]
Whenever $n_{\min}(t)\ge n_\mu$, we have
$r_t^{\max}\le\eta_\mu/16$.
On $\mathcal E_T$, the empirical Pareto set
is then correct, and
\[
\widehat\eta_t
\ge\eta_\mu-2r_t^{\max}
\ge8r_t^{\max}.
\]
Thus, the evidence gate is activated once
every arm has received $n_\mu$ observations,
unless it has already been released.

By the assumed condition $\rho_T\ge n_\mu$,
the quota satisfies
\[
h_{t,T}\ge n_\mu,
\qquad
t\ge\max\{2K,\lceil\log T\rceil\}.
\]
Before release, whenever $n_{\min}(t)<n_\mu$,
the quota branch therefore selects a
least-sampled arm with probability $p_g$.
Each such selection decreases
\[
\sum_{i=1}^K[n_\mu-N_i(t)]_+
\]
by one. Since every arm is initially sampled
twice, at most $K(n_\mu-2)$ successful quota
selections are required.

Define
\[
s=
\max\{2K,\lceil\log T\rceil\}
+
\left\lceil
\frac{2}{p_g}
\bigl[K(n_\mu-2)+8\Lambda_T\bigr]
\right\rceil+1.
\]
Coupling the quota decisions with independent
Bernoulli variables of parameter $p_g$,
a Chernoff bound gives
\[
\Prob\left\{
\operatorname{Bin}\left(
\left\lceil
\frac{2}{p_g}
[K(n_\mu-2)+8\Lambda_T]
\right\rceil,p_g
\right)
<K(n_\mu-2)
\right\}
\le e^{-2\Lambda_T}.
\]

By the definition of $C_\mu$, we have
$s\le C_\mu\Lambda_T$.
Consequently, there exists an event
$\mathcal G_T$ such that the gate is
released by time $s$ on $\mathcal G_T$, and
\begin{equation}\label{eq:egatt-good}
\Prob_\mu(\mathcal G_T^c)
\le\frac1{2T^6}+e^{-2\Lambda_T}
\le T^{-6}.
\end{equation}

\paragraph{Arm selection after release.}
Let $\tau\le s$ be the release time on
$\mathcal G_T$.
By the gate condition,
$\widehat\eta_\tau\ge8r_\tau^{\max}$.
On $\mathcal E_T$, this certifies the true
Pareto set and gives
\[
\eta_\mu
\ge\widehat\eta_\tau-2r_\tau^{\max}
\ge6r_\tau^{\max}.
\]
Since $r_t^{\max}$ is nonincreasing,
the empirical Pareto set remains correct
for every $t\ge\tau$, with
\[
\widehat\eta_t
\ge\eta_\mu-2r_t^{\max}
\ge4r_t^{\max}.
\]

We now compare the gains in \eqref{eq:gain}.
For a dominated arm $i$, its empirical
domination margin satisfies
\[
\widehat\Delta_{t,i}
\ge\widehat\eta_t
\ge4r_t^{\max}.
\]
Moreover, any alternative that changes
the Pareto membership of arm $i$ must
perturb at least one mean vector involved
in an empirical domination relation
by at least $\widehat\eta_t/2$
in the infinity norm.
By the definition of $Z_{t,i}$,
\[
Z_{t,i}
\ge
\frac{n_{\min}(t)\widehat\eta_t^2}
{8\sigma^2}
\ge4\Lambda_T.
\]
Thus, recalling
$R_{t,i}=\sqrt{N_i(t)}\widehat\Delta_{t,i}/\sigma$,
\[
R_{t,i}\sqrt{1+Z_{t,i}}
\ge8\sqrt2\,\Lambda_T.
\]

Since $c_\tau=1$, the unnormalized gain
of every dominated arm satisfies
\[
\widetilde G_{t,i}
\le
e^{-8\sqrt2\Lambda_T}
\log\left(1+\frac1{N_i(t)}\right)
<\frac1{T+1}.
\]
For a Pareto optimal arm $j$, however,
$\widehat\Delta_{t,j}=0$, and hence
\[
\widetilde G_{t,j}
=
\log\left(1+\frac1{N_j(t)}\right)
\ge\frac1{T+1}.
\]
Since normalization does not change
the ordering of gains, all Pareto optimal
arms have larger gains than dominated arms.

When $m\ge2$, the two largest gains
therefore correspond to the two
least-sampled Pareto optimal arms.
The adaptive probability in
\eqref{eq:beta} assigns probability
at least $0.70$ to the arm with
the smaller count.
When $m=1$, the rule selects the
unique Pareto optimal arm with
probability one.

It follows that no dominated arm is
selected after release on $\mathcal G_T$.
In particular,
\[
B_T\le s
\qquad\text{on }\mathcal G_T.
\]
Since $B_T\le T$ otherwise,
\begin{equation}\label{eq:egatt-B-moment}
\E_\mu B_T^4
\le s^4+T^4\Prob_\mu(\mathcal G_T^c)
\le C_\mu^4\Lambda_T^4+T^{-2}.
\end{equation}

\paragraph{Moment bound for the imbalance.}
Suppose $m\ge2$.
Starting from the actual counts at the
deterministic time $s$, construct an
auxiliary process that always selects
the two least-sampled true Pareto optimal
arms, using the adaptive probability
in \eqref{eq:beta}.

Couple its randomization with EG-ATT.
On $\mathcal G_T$, the two processes
coincide after $s$. Importantly, the
auxiliary process is defined on every
trajectory, so its transition probabilities
can be analyzed without conditioning
on $\mathcal G_T$.

Let $x_i(t)$ denote the centered count
of Pareto optimal arm $i$ in this process,
and write
\[
D_t=
\left(\sum_{i\in\Astar}x_i(t)^2\right)^{1/2}.
\]
Order these counts as
$x_1\le x_2\le\cdots\le x_m$.
Since the smaller count is selected
with probability at least $p_\star=0.70$,
\[
\E[x_{A_{t+1}}(t)\mid\mathcal H_t]
\le
p_\star x_1+(1-p_\star)x_2,
\]
where $\mathcal H_t$ denotes the
auxiliary process history.

Using $\sum_i x_i=0$, we have
\[
x_2\le-\frac{x_1}{m-1},
\qquad
D_t\le-\sqrt{m(m-1)}\,x_1.
\]
Consequently,
\[
\begin{aligned}
\E[x_{A_{t+1}}(t)\mid\mathcal H_t]
&\le
\frac{mp_\star-1}{m-1}x_1\\
&\le-\frac{D_t}{2m},
\end{aligned}
\]
where the last inequality holds for
every $m\ge2$ when $p_\star=0.70$.

Each selection changes the centered
count vector by $e_i-\mathbf1/m$.
Therefore,
\[
D_{t+1}^2-D_t^2
=
2x_{A_{t+1}}(t)+1-\frac1m,
\]
which gives the drift bound
\begin{equation}\label{eq:egatt-drift}
\E[D_{t+1}^2-D_t^2\mid\mathcal H_t]
\le-\frac{D_t}{m}+1-\frac1m.
\end{equation}

We next establish a uniform fourth-moment
bound. For $D_t\ge2(m-1)$, concavity of
the square root and \eqref{eq:egatt-drift}
give
\[
\E[D_{t+1}-D_t\mid\mathcal H_t]
\le-\frac1{4m}.
\]
Since $|D_{t+1}-D_t|\le1$,
Hoeffding's lemma, with $\theta=1/(4m)$,
implies
\[
\E[e^{\theta D_{t+1}}\mid\mathcal H_t]
\le
e^{-1/(32m^2)}e^{\theta D_t}.
\]
When $D_t<2(m-1)$, the left-hand side
is bounded by $e^{\theta(2m-1)}$.
Combining these cases and iterating
from time $s$, we obtain
\[
\E e^{\theta D_T}
\le
\exp\left(
\frac{s}{4m}-\frac{T-s}{32m^2}
\right)
+
\frac{e^{(2m-1)/(4m)}}
{1-e^{-1/(32m^2)}}.
\]

The sufficient horizon condition yields
\[
T-s\ge8ms.
\]
Hence the first term is at most one,
and the exponential moment is bounded
by a constant depending only on $m$.
Using $x^4\le24e^{\theta x}/\theta^4$
for $x\ge0$, we conclude that
\begin{equation}\label{eq:egatt-aux-moment}
\E D_T^4\le M_m,
\end{equation}
where $M_m<\infty$ depends only on $m$.

On $\mathcal G_T$, the actual imbalance
satisfies $\Phi_T=D_T^2$.
On its complement, $\Phi_T^2\le T^4$.
Together with \eqref{eq:egatt-good}
and \eqref{eq:egatt-aux-moment},
this gives
\begin{equation}\label{eq:egatt-Phi-moment}
\E_\mu\Phi_T^2
\le M_m+T^{-2}.
\end{equation}

\paragraph{Moments of fairness.}
Recall the decomposition
\[
\F_T=
-\log\left(1-\frac{B_T}{T}\right)
+D_{\mathrm{KL}}(U_m\|q_T),
\]
where $q_T$ is the empirical allocation
conditioned on Pareto optimal pulls.

Let $H=T-s$ and
$a=-\log(1-s/T)$.
On $\mathcal G_T$, we have $B_T\le s$.
If the auxiliary imbalance additionally
satisfies $D_T<H/(2m)$, every
Pareto optimal arm has at least half
its balanced allocation.
The inequality between KL divergence
and chi-squared divergence gives
\[
D_{\mathrm{KL}}(U_m\|q_T)
\le\frac{2m\Phi_T}{(T-B_T)^2}
\le\frac{2mD_T^2}{H^2}.
\]
Therefore,
\[
\F_T\le a+\frac{2mD_T^2}{H^2}
\]
on this event.

By Markov's inequality and
\eqref{eq:egatt-aux-moment}, the
probability that this bound does not
apply is at most
\[
\delta_T=
T^{-6}+\frac{16m^4M_m}{H^4}.
\]
On the complementary event, we use
the general bound
$\F_T\le\log(T/m)$.

Taking expectations and using
$\E D_T^2\le\sqrt{M_m}$ gives
\[
\E_\mu\F_T
\le
a+\frac{2m\sqrt{M_m}}{H^2}
+\delta_T\log(T/m).
\]
Similarly, since
$\Var_\mu(\F_T)\le\E_\mu\F_T^2$,
\[
\Var_\mu(\F_T)
\le
2a^2+\frac{8m^2M_m}{H^4}
+\delta_T\log^2(T/m).
\]

Finally, the horizon condition implies
$s\le C_\mu\Lambda_T\le T/2$, so
\[
a\le\frac{2C_\mu\Lambda_T}{T},
\qquad
H\ge\frac T2,
\qquad
\delta_T
\le T^{-6}+\frac{256m^4M_m}{T^4}.
\]
Substitution yields
\[
\E_\mu\F_T
\le\frac{2C_\mu\Lambda_T}{T}
+\frac{D}{T^2},
\]
and
\[
\Var_\mu(\F_T)
\le\frac{8C_\mu^2\Lambda_T^2}{T^2}
+\frac{D\log^2T}{T^4},
\]
for a sufficiently large constant $D$
depending only on $m$.
Increasing $D$ if necessary also gives
the claimed bound on $\E_\mu\Phi_T^2$.

\paragraph{Case $m=1$.}
When the Pareto set contains a single
arm, $\Phi_T=0$.
On $\mathcal G_T$, the unique Pareto optimal
arm is selected after release, so
$\F_T\le a$.
Consequently,
\[
\E_\mu\F_T
\le a+T^{-6}\log T,
\qquad
\Var_\mu(\F_T)
\le a^2+T^{-6}\log^2T.
\]
These satisfy the stated bounds
with a suitable constant $D$.

Together with \eqref{eq:egatt-B-moment},
this completes the proof.
\end{proof}

\section{Experimental Details}
\label{app:fairness-instance-details}

\subsection{Reward Instances}
The three instances for each reward model are listed below.
Arm indices are consistent across separation margins.
The Pareto optimal arms are $\{1,2,3,4\}$ for Bernoulli rewards
and $\{1,2,3\}$ for Gaussian rewards. The Gaussian rewards
have standard deviation $\sigma=0.12$.

\begin{table}[h]
\centering
\small
\caption{Mean reward vectors for the bernoulli instances.}
\label{tab:fairness-instance-bern}
\begin{tabular}{lccc}
\toprule
Arm & $\eta_\mu=0.005$ & $\eta_\mu=0.01$ & $\eta_\mu=0.1$ \\
\midrule
1 & $(0.525,0.5)$ & $(0.55,0.5)$ & $(0.6,0.1)$ \\
2 & $(0.515,0.505)$ & $(0.53,0.51)$ & $(0.4,0.2)$ \\
3 & $(0.51,0.52)$ & $(0.52,0.54)$ & $(0.3,0.5)$ \\
4 & $(0.5,0.535)$ & $(0.5,0.57)$ & $(0.1,0.8)$ \\
5 & $(0.505,0.505)$ & $(0.51,0.51)$ & $(0.2,0.2)$ \\
6 & $(0.5,0.5)$ & $(0.5,0.5)$ & $(0.1,0.1)$ \\
\bottomrule
\end{tabular}
\end{table}

\begin{table}[t]
\centering
\small
\caption{Mean reward vectors for the gaussian instances.}
\label{tab:fairness-instance-gaussian}
\begin{tabular}{lccc}
\toprule
Arm & $\eta_\mu=0.005$ & $\eta_\mu=0.02$ & $\eta_\mu=0.1$ \\
\midrule
1 & $(0.5425,0.53)$ & $(0.67,0.62)$ & $(0.8,0.55)$ \\
2 & $(0.53,0.5425)$ & $(0.62,0.67)$ & $(0.55,0.8)$ \\
3 & $(0.535,0.535)$ & $(0.64,0.64)$ & $(0.65,0.65)$ \\
4 & $(0.5275,0.5275)$ & $(0.61,0.61)$ & $(0.5,0.5)$ \\
5 & $(0.525,0.535)$ & $(0.6,0.64)$ & $(0.45,0.65)$ \\
\bottomrule
\end{tabular}
\end{table}

\subsection{Baseline Implementations}
We include the five baselines compared in the main text:
Pareto-UCB1, Pareto Thompson sampling (PTS), Width-guided UCB,
Pareto-KG, and Annealing-Pareto.
Pareto-UCB1 uses the optimistic confidence radius without a Gaussian
noise multiplier. Bernoulli PTS uses independent $\operatorname{Beta}(1,1)$
priors, while Gaussian PTS uses normal mean posteriors.
Width-guided UCB uses the theorem-matched coefficient $c=2$.
Pareto-KG uses a posterior-uncertainty plug-in and is therefore
an adaptation of the original method. Annealing-Pareto uses decay
$0.4$ for Bernoulli rewards and the fixed benchmark choice $0.97$
for Gaussian rewards. Applying confidence rules originally developed
for bounded rewards to Gaussian observations is also an adaptation.

\subsection{Tuning Details.}
For each reward family, separation margin, and horizon,
we tuned ATT and EG-ATT separately under two Bayesian-optimization (BO)
objectives: the mean-plus-standard-deviation objective (M) and the
log-scale objective (L), as defined in the main text.
Each search evaluated 20 configurations using 28 pilot runs,
compared leading configurations using 90 validation runs,
and evaluated the selected configuration on 240 independent held-out runs.
Tables~\ref{tab:tuned-bo-bern} and~\ref{tab:tuned-bo-gauss}
report the selected parameters together with the held-out
mean and variance of terminal unfairness for every setting.
The parameter vectors are
$\boldsymbol\theta=(c_\tau,p_0,\delta,\lambda,p_-,p_+)$,
for the temperature, leader-probability intercept, count-deficit
strength, denominator scale, and clipping limits, and
$\boldsymbol\omega=(q_0,p_g,a_\rho,a_g)$,
for the quota coefficient, quota probability, reserve multiplier,
and gate multiplier; the latter are used only by EG-ATT.
Numerical parameter values are rounded for display.
In all 48 selected configurations, the observed gate-release fraction
was zero, so these finite-horizon EG-ATT results reflect quota protection
without gate release.

\begin{table}[h]
\centering
\scriptsize
\setlength{\tabcolsep}{2.5pt}
\caption{Selected BO parameters and held-out unfairness for Bernoulli rewards.
M denotes mean+SD tuning; L denotes log-scale tuning.}
\label{tab:tuned-bo-bern}
\resizebox{\linewidth}{!}{%
\begin{tabular}{rrllllrr}
\toprule
$\eta_\mu$ & $T$ & Method & BO & $\boldsymbol\theta$ & $\boldsymbol\omega$ & Mean & Variance \\
\midrule
0.005 & 5,000 & ATT & M & $(2.87,0.658,0.010,0.155,0.452,0.994)$ & -- & $0.3655$ & $0.004249$ \\
0.005 & 5,000 & ATT & L & $(2.88,0.597,0.200,0.669,0.317,0.994)$ & -- & $0.3651$ & $0.004257$ \\
0.005 & 5,000 & EG-ATT & M & $(2.90,0.589,0.267,0.523,0.325,0.994)$ & $(107.52,0.552,0.494,9.67)$ & $0.3719$ & $0.001447$ \\
0.005 & 5,000 & EG-ATT & L & $(1.88,0.581,0.028,0.189,0.414,0.912)$ & $(155.33,0.762,0.673,8.62)$ & $0.4048$ & $2.7\times10^{-6}$ \\
\addlinespace[2pt]
0.005 & 50,000 & ATT & M & $(1.57,0.578,0.280,0.026,0.482,0.920)$ & -- & $0.2495$ & $0.01683$ \\
0.005 & 50,000 & ATT & L & $(3.00,0.801,0.266,0.116,0.507,0.936)$ & -- & $0.2758$ & $0.004433$ \\
0.005 & 50,000 & EG-ATT & M & $(1.40,0.782,0.417,0.276,0.536,0.983)$ & $(8.37,0.949,0.750,9.82)$ & $0.2488$ & $0.01999$ \\
0.005 & 50,000 & EG-ATT & L & $(2.89,0.622,0.088,0.203,0.464,0.926)$ & $(159.27,0.633,0.347,9.22)$ & $0.273$ & $0.004707$ \\
\addlinespace[2pt]
0.01 & 5,000 & ATT & M & $(2.18,0.718,0.261,0.042,0.488,0.937)$ & -- & $0.3041$ & $0.009042$ \\
0.01 & 5,000 & ATT & L & $(2.95,0.558,0.169,0.044,0.478,0.986)$ & -- & $0.3099$ & $0.004766$ \\
0.01 & 5,000 & EG-ATT & M & $(2.03,0.869,0.107,0.538,0.390,0.980)$ & $(2.17,0.592,0.666,7.18)$ & $0.302$ & $0.01001$ \\
0.01 & 5,000 & EG-ATT & L & $(2.91,0.721,0.468,0.190,0.437,0.932)$ & $(148.17,0.918,0.551,7.05)$ & $0.4053$ & $7.5\times10^{-10}$ \\
\addlinespace[2pt]
0.01 & 50,000 & ATT & M & $(1.00,0.700,0.250,0.150,0.450,0.950)$ & -- & $0.1368$ & $0.02818$ \\
0.01 & 50,000 & ATT & L & $(2.36,0.567,0.084,0.068,0.504,0.910)$ & -- & $0.1578$ & $0.004353$ \\
0.01 & 50,000 & EG-ATT & M & $(1.77,0.669,0.130,0.148,0.484,0.960)$ & $(35.36,0.504,0.930,10.94)$ & $0.1415$ & $0.006578$ \\
0.01 & 50,000 & EG-ATT & L & $(2.32,0.769,0.029,0.561,0.525,0.973)$ & $(2.62,0.651,0.785,11.58)$ & $0.1566$ & $0.004438$ \\
\addlinespace[2pt]
0.1 & 5,000 & ATT & M & $(1.00,0.700,0.250,0.150,0.450,0.950)$ & -- & $0.0246$ & $0.002773$ \\
0.1 & 5,000 & ATT & L & $(1.00,0.700,0.250,0.150,0.450,0.950)$ & -- & $0.0246$ & $0.002773$ \\
0.1 & 5,000 & EG-ATT & M & $(0.739,0.693,0.244,0.428,0.538,0.966)$ & $(2.38,0.784,0.449,8.96)$ & $0.02273$ & $6.55\times10^{-5}$ \\
0.1 & 5,000 & EG-ATT & L & $(1.00,0.700,0.250,0.150,0.450,0.950)$ & $(96.00,0.700,1.00,8.00)$ & $0.265$ & $1.93\times10^{-14}$ \\
\addlinespace[2pt]
0.1 & 50,000 & ATT & M & $(0.517,0.869,0.049,0.656,0.349,0.970)$ & -- & $0.006803$ & $0.005705$ \\
0.1 & 50,000 & ATT & L & $(0.426,0.895,0.493,0.385,0.441,0.998)$ & -- & $0.01402$ & $0.01824$ \\
0.1 & 50,000 & EG-ATT & M & $(0.443,0.893,0.483,0.666,0.534,0.977)$ & $(3.62,0.567,0.118,6.99)$ & $0.002758$ & $7.72\times10^{-7}$ \\
0.1 & 50,000 & EG-ATT & L & $(0.322,0.893,0.416,0.524,0.313,0.996)$ & $(3.73,0.782,0.270,9.41)$ & $0.003614$ & $2.08\times10^{-7}$ \\
\bottomrule
\end{tabular}%
}
\end{table}
\begin{table}[t]
\centering
\scriptsize
\setlength{\tabcolsep}{2.5pt}
\caption{Selected BO parameters and held-out unfairness for Gaussian rewards.
M denotes mean+SD tuning; L denotes log-scale tuning.}
\label{tab:tuned-bo-gauss}
\resizebox{\linewidth}{!}{%
\begin{tabular}{rrllllrr}
\toprule
$\eta_\mu$ & $T$ & Method & BO & $\boldsymbol\theta$ & $\boldsymbol\omega$ & Mean & Variance \\
\midrule
0.005 & 5,000 & ATT & M & $(2.20,0.586,0.221,0.080,0.460,0.915)$ & -- & $0.3188$ & $0.009786$ \\
0.005 & 5,000 & ATT & L & $(2.92,0.560,0.446,0.580,0.307,0.993)$ & -- & $0.3468$ & $0.007814$ \\
0.005 & 5,000 & EG-ATT & M & $(0.758,0.813,0.068,0.351,0.450,0.921)$ & $(42.86,0.943,0.004,6.00)$ & $0.2826$ & $0.01957$ \\
0.005 & 5,000 & EG-ATT & L & $(2.05,0.662,0.389,0.172,0.386,0.922)$ & $(155.36,0.580,0.806,8.67)$ & $0.4923$ & $4.89\times10^{-5}$ \\
\addlinespace[2pt]
0.005 & 50,000 & ATT & M & $(1.00,0.700,0.250,0.150,0.450,0.950)$ & -- & $0.07131$ & $0.0168$ \\
0.005 & 50,000 & ATT & L & $(1.80,0.697,0.270,0.036,0.440,0.939)$ & -- & $0.09463$ & $0.001683$ \\
0.005 & 50,000 & EG-ATT & M & $(1.00,0.700,0.250,0.150,0.450,0.950)$ & $(96.00,0.700,1.00,8.00)$ & $0.06995$ & $0.004597$ \\
0.005 & 50,000 & EG-ATT & L & $(0.823,0.677,0.068,0.082,0.479,0.921)$ & $(65.22,0.881,0.665,9.33)$ & $0.06028$ & $0.005503$ \\
\addlinespace[2pt]
0.02 & 5,000 & ATT & M & $(1.00,0.700,0.250,0.150,0.450,0.950)$ & -- & $0.05165$ & $0.01231$ \\
0.02 & 5,000 & ATT & L & $(1.00,0.700,0.250,0.150,0.450,0.950)$ & -- & $0.05165$ & $0.01231$ \\
0.02 & 5,000 & EG-ATT & M & $(0.707,0.870,0.361,0.545,0.477,0.912)$ & $(4.14,0.791,0.635,11.16)$ & $0.03897$ & $0.001997$ \\
0.02 & 5,000 & EG-ATT & L & $(0.784,0.801,0.108,0.102,0.430,0.926)$ & $(78.86,0.535,0.996,9.69)$ & $0.2196$ & $6.07\times10^{-15}$ \\
\addlinespace[2pt]
0.02 & 50,000 & ATT & M & $(1.00,0.700,0.250,0.150,0.450,0.950)$ & -- & $0.006584$ & $7.45\times10^{-6}$ \\
0.02 & 50,000 & ATT & L & $(1.00,0.700,0.250,0.150,0.450,0.950)$ & -- & $0.006584$ & $7.45\times10^{-6}$ \\
0.02 & 50,000 & EG-ATT & M & $(0.416,0.896,0.043,0.556,0.497,0.950)$ & $(3.29,0.578,0.303,8.77)$ & $0.004619$ & $3.22\times10^{-6}$ \\
0.02 & 50,000 & EG-ATT & L & $(0.257,0.894,0.484,0.329,0.510,0.971)$ & $(12.30,0.935,0.481,11.44)$ & $0.00686$ & $1.34\times10^{-6}$ \\
\addlinespace[2pt]
0.1 & 5,000 & ATT & M & $(0.405,0.712,0.119,0.151,0.323,0.954)$ & -- & $0.001771$ & $1.94\times10^{-6}$ \\
0.1 & 5,000 & ATT & L & $(0.703,0.551,0.040,0.651,0.465,0.956)$ & -- & $0.002576$ & $1.12\times10^{-6}$ \\
0.1 & 5,000 & EG-ATT & M & $(0.157,0.796,0.269,0.239,0.362,0.988)$ & $(2.09,0.763,0.035,9.25)$ & $0.007226$ & $4.92\times10^{-16}$ \\
0.1 & 5,000 & EG-ATT & L & $(0.274,0.582,0.155,0.326,0.370,0.987)$ & $(2.32,0.698,0.026,11.96)$ & $0.00723$ & $2.19\times10^{-9}$ \\
\addlinespace[2pt]
0.1 & 50,000 & ATT & M & $(0.584,0.796,0.479,0.079,0.533,0.995)$ & -- & $0.000296$ & $1.47\times10^{-8}$ \\
0.1 & 50,000 & ATT & L & $(0.242,0.589,0.496,0.597,0.319,0.950)$ & -- & $0.0001842$ & $2.03\times10^{-7}$ \\
0.1 & 50,000 & EG-ATT & M & $(0.216,0.869,0.012,0.328,0.515,0.970)$ & $(2.13,0.895,0.105,10.87)$ & $0.001681$ & $2.42\times10^{-20}$ \\
0.1 & 50,000 & EG-ATT & L & $(0.190,0.659,0.231,0.410,0.489,0.923)$ & $(2.43,0.851,0.044,11.95)$ & $0.001281$ & $6.68\times10^{-12}$ \\
\bottomrule
\end{tabular}%
}
\end{table}

Lastly, we observe that the terminating condition is never met
in the reported EG-ATT runs. Thus, these experiments isolate
the effect of quota protection at the tested horizons. This may
suggest that the terminating condition is conservative or that
longer horizons are needed for it to be met. Further investigation
is beyond the scope of this paper but may be of independent interest.

\end{document}